\documentclass[twoside]{article}
\usepackage[accepted]{aistats2016}
\usepackage{amsmath,amsfonts,amsthm,amssymb}
\usepackage{algpseudocode}
\usepackage{algorithm}
\usepackage{footnote}
\usepackage{color,xcolor}
\usepackage{graphicx,subfigure}
\usepackage[english]{babel}
\usepackage[ulem=normalem]{changes}
\usepackage[title]{appendix}
\usepackage{mathrsfs}
\usepackage{tabularx}

\usepackage{color,xcolor,colortbl}
\definecolor{LightCyan}{rgb}{0.88,1,1}

\newcommand{\vct}{\boldsymbol }
\newcommand{\mat}{\mathbf}

\newcommand{\nml}{\mathcal{N}}

\renewcommand{\span}{\mathrm{span}}
\newcommand{\argmin}{\mathrm{argmin}}

\newcommand{\conv}{\mathrm{conv}}

\newcommand{\range}{\mathrm{range}}

\newcommand{\aff}{\mathrm{aff}}
\renewcommand{\range}{\mathrm{Range}}

\def\R{\mathbb{R}}

\def\cS{\mathcal{S}}

\newtheorem{thm}{Theorem}[section]
\newtheorem{lem}{Lemma}[section]

\newtheorem{prop}{Proposition}[section]
\newtheorem{asmp}{Assumption}[section]
\newtheorem{defn}{Definition}[section]

\numberwithin{equation}{section}

\begin{document}

%

%

\twocolumn[

\aistatstitle{ Graph Connectivity in Noisy Sparse Subspace Clustering}

\aistatsauthor{ Yining Wang, Yu-Xiang Wang and Aarti Singh }

\aistatsaddress{ Machine Learning Department, School of Computer Science, Carnegie Mellon University } 

]

\begin{abstract}
   Subspace clustering is the problem of clustering data points into a union of low-dimensional linear/affine subspaces.
   It is the mathematical abstraction of many important problems in computer vision, image processing and machine learning.
   A line of recent work \cite{ssc,sc-geometric,wang2015noisy,robust-ssc} provided strong theoretical guarantee for sparse subspace clustering \cite{ssc},
   the state-of-the-art algorithm for subspace clustering,
   on both noiseless and noisy data sets.
   It was shown that under mild conditions, with high probability no two points from different subspaces are clustered together.
   Such guarantee, however, is \emph{not} sufficient for the clustering to be correct, due to the notorious ``graph connectivity problem'' \cite{graph-connectivity}. 
   In this paper, we investigate the graph connectivity problem for \emph{noisy} sparse subspace clustering
   and show that a simple post-processing procedure is capable of delivering consistent clustering under certain ``general position" or ``restricted eigenvalue" assumptions.
   We also show that our condition is almost tight with adversarial noise perturbation by constructing a counter-example.
   These results provide the first \emph{exact clustering} guarantee of noisy SSC for subspaces of dimension greater then $3$.
\end{abstract}

\section{INTRODUCTION}

The problem of subspace clustering originates from numerous applications in computer vision and image processing, where there are either physical laws or empirical evidence that ensure a given set of data points to form a union of linear or affine subspaces.
 Such data points could be feature trajectories of rigid moving objects captured by an affine camera~\cite{ssc}, articulated moving parts of a human body~\cite{yan2006general}, illumination of different convex objects under Lambertian model~\cite{ho2003clustering} and so on. 
 Subspace clustering is also more generically used in agnostic learning of the best linear mixture structures in the data. 
 For instance, it is used for images/video compression~\cite{hong2006multiscale}, hybrid system identification, disease identification~\cite{mcwilliams2014subspace} as well as modeling social network communities~\cite{chen2014clustering}, studying privacy in movie recommendations~\cite{zhang2012guess} and inferring router network topology~\cite{eriksson2012high}.


There is rich literature on algorithmic and theoretical analysis of subspace clustering \cite{ssc,liu2010robust,heckel2013subspace,park2014greedy}.
Among the many algorithms, sparse subspace clustering (SSC)~\cite{ssc} is arguably the most well-studied due to its elegant formulation, strong empirical performance and provable guarantees to work under relatively weak conditions. The algorithm involves constructing a sparse linear representation of each data point using the remaining dataset as a dictionary. This approach embeds the relationship of the data points into a sparse graph and the intuition is that the data points are likely to choose \emph{only} those points on the same subspace to linearly represent itself. Then clustering  can be obtained by finding connected components of the graph, or more robustly, using spectral clustering \cite{ssc}.

Assuming data lie exactly or approximately on a union of linear subspaces,
\footnote{affine subspaces are handled by augmenting $1$ to every data point.}
 it is shown in \cite{ssc,sc-geometric,wang2015noisy,robust-ssc} that under certain separation conditions, this embedded graph will have no edges between any two points in different subspaces. This criterion of success is referred to as the ``Self-Expressiveness Property (SEP)'' \cite{ssc,wang2015noisy} and ``Subspace Detection Property (SDP)'' \cite{sc-geometric}. The drawback is that there is no guarantee that the vertices within one cluster form a connected component. Therefore, the solution may potentially over segment the data points. This subtle point was originally raised and partially addressed in \cite{graph-connectivity}, reaching an answer that when data are noiseless and intrinsic subspace dimension $d\leq 3$, such over-segmentation will not occur as long as all points within the same subspace are in general position; but when $d\geq4$, a counter example was provided, showing that this weak ``general position'' condition is no longer sufficient.
 
 In this paper, we revisit the graph connectivity problem for \emph{noisy} sparse subspace clustering.
 Inspired by the post-merging step presented in \cite{ssc} for noiseless data,
 we propose in this paper a variant of noisy sparse subspace clustering \cite{noisy-ssc} that provably produces perfect clustering with high probability, 
 under certain ``general position" or ``restricted eigenvalue" assumptions.
 We also provide a counter-example to show that our derived success conditions are almost tight under the adversarial noise perturbation model.
 This is the first time a subspace clustering algorithm is proven to give correct clustering under no statistical assumptions on data corrupted by noise.
 To the best of our knowledge, this is also the first guarantee for Lasso that lower bounds the number of discoveries,
 which might be of independent interest for other problems that uses Lasso as a subroutine.

\subsection{Problem setup and notations}\label{subsec:notation}
For a vector $\vct x$ we use $\|\vct x\|_p = (\sum_i{\vct x_i^p})^{1/p}$ to denote its $p$-norm.
If $p$ is not explicitly specified then the 2-norm is used. The noiseless data matrix is denoted as $\mat X=(\vct x_1,\cdots,\vct x_N)\in\mathbb R^{n\times N}$ where $n$ is the ambient dimension and $N$ denotes the number of data points available.
Each data point $\vct x_i\in\mathbb R^n$ is normalized so that it has unit two norm.
We use $\mathcal S\subseteq\mathbb R^n$ to denote a low-dimensional linear subspace in $\mathbb R^n$ and $\mat S\in\mathbb R^{n\times d}$
for an orthonormal basis of $\mathcal S$, where $d$ is the intrinsic rank of $\mathcal S$.
For subspace clustering it is assumed that each data point $\vct x_i$ lies on a union of underlying subspaces $\bigcup_{\ell=1}^L{\mathcal S^{(\ell)}}$
with intrinsic dimensions $d_1,\cdots,d_L < n$.
We use $z_1,\cdots,z_N\in\{1,2,\cdots,L\}$ to denote the ground truth cluster assignments of each data point in $\mat X$
and $\mat X^{(\ell)}=\{\vct x_i\in\mat X: z_i=\ell\}$ to denote all data points in the $\ell$th cluster.
Define $d(\vct x_i,\mathcal S)=\inf_{\vct y\in\mathcal S}{\|\vct x-\vct y\|_2}$ as the distance between a point $\vct x$ and a linear subspace $\mathcal S$.
Since $\mat X$ is noiseless, we have $d(\vct x_i,\mathcal S^{(z_i)}) = 0$.
The objective of subspace clustering is to recover $\{\mathcal S^{(\ell)}\}_{\ell=1}^L$ and $\{z_i\}_{i=1}^N$ up to permutations.

Under the fully deterministic data model \cite{sc-geometric} no additional stochastic model is assumed on either the underlying subspaces or the data points.
For noisy subspace clustering we observe a noise-perturbed matrix $\mat Y=(\vct y_1,\cdots,\vct y_N)\in\mathbb R^{n\times N}$ where $\vct y_i=\vct x_i+\vct\varepsilon_i$.
The noise variables $\{\vct\varepsilon_i\}_{i=1}^N$ considered previously can be either deterministic (i.e., adversarial) or stochastic (e.g., Gaussian white noise) \cite{wang2015noisy,robust-ssc}.

Given ground-truth clustering $\{z_i\}_{i=1}^N\subseteq\{1,\cdots L\}$, a similarity graph $\mat C\in\mathbb R^{N\times N}$ satisfies \emph{Self-Expressiveness Property} (SEP, \cite{ssc})
if $|\mat C_{ij}| > 0$ implies $z_i=z_j$.
Note that the reverse is not necessarily true.
That is, $z_i=z_j$ does \emph{not} imply $|\mat C_{ij}| > 0$.

\section{RELATED WORK}

The pursuit of provable subspace clustering methods has seen much progress recently. Theoretical guarantees for several algorithms have been established in many regimes.
 At times it may get confusing what these results actually mean. In this section, we first review the different assumptions and claims in the literature and then pinpoint what our contributions are.

Table~\ref{tab:subspace} lists the hierarchies of assumptions on the subspaces. Each row is weaker than its previous row.
Except for the independent subspace assumption, which on its own is sufficient, results for more general models typically require additional conditions on the subspaces and data points in each subspaces. For instance, the ``semi-random model'' assumes data points to be drawn i.i.d.~uniformly at random from the unit sphere in each subspace and the more generic ``deterministic model'' places assumptions on the radius of the smallest inscribing sphere of the symmetric polytope spanned by data points \cite{sc-geometric} or the smallest non-zero singular value of the data matrix \cite{wang2013lrssc}.
Related theoretical guarantees of subspace clustering algorithms in the literature are summarized in Table~\ref{tab:reviews} where the assumptions about subspaces are denoted with capital letters ``A, B, C''; different noise settings are referred to using lowercase letters ``a,b,c'' in Table~\ref{tab:datapoint_assumptions}. Results that are applicable to SSC are highlighted.

\begin{table}[t]
	\centering
	\caption{The hierarchies of assumptions on the subspaces. 
	$A$: independent subspaces;
	$B$: disjoint subspaces\textsuperscript{*};
	$C$: overlapping subspaces\textsuperscript{*}.
	Note that $A \subset B \subset C$. 
	Superscript $^*$ indicates that additional separation conditions are needed.}
	\vskip 0.1in
	\begin{tabularx}{0.45\textwidth}{l|X}
		A & $\dim\left[\cS_1\otimes...\otimes \cS_L\right] = \sum_{\ell=1}^L \dim\left[\cS_\ell\right]  $.  \\\hline
		B &  $\cS_\ell\cap \cS_{\ell'} =\mathbf{0}$ for all $\{(\ell,\ell')|\ell\neq \ell'\}$.\\\hline
		C & $\dim(\cS_\ell\cap \cS_{\ell'}) < \min \left\{\dim(\cS_\ell), \dim(\cS_{\ell'})\right\}$ \newline
		for all $\{(\ell,\ell')|\ell\neq \ell'\}$.\\
	\end{tabularx}
\label{tab:subspace}
	
	\caption{Reference of assumptions on data points. Columns correspond to data point generation assumptions and rows correspond to different noise regimes. }\label{tab:datapoint_assumptions}	
	\vskip 0.1in
	\begin{tabular}{c|c|c}
		& 1. Semi-Random & 2.Deterministic \\
		\hline a. noiseless & $\vct\varepsilon_i=\vct 0$& $\vct\varepsilon_i=\vct 0$\\
		\hline b. stochastic& $\vct\varepsilon_i\sim \nml(0,\sigma^2\mat I)$& $\vct\varepsilon_i\sim\nml(0,\sigma^2\mat I)$\\
		\hline c. adversarial& $\|\vct\varepsilon_i\|_2\leq \xi$& $\|\vct\varepsilon_i\|_2\leq \xi$\\
	\end{tabular}
\end{table}

As we can see from the second column of Table~\ref{tab:reviews}, SEP guarantees have been quite exhaustively studied and now we understand very well the conditions under which it holds. Specifically, most of the results are now near optimal under the semi-random model: SEP holds in cases even when different subspaces substantially overlap, have canonical angles near $0$, the dimension of the subspaces being linear in the ambient dimension, or the number of subspaces to be clustered is exponentially large \cite{sc-geometric,wang2015noisy,robust-ssc}.
In addition, the above results also hold robustly under a small amount of arbitrary perturbation or a large amount of stochastic noise \cite{wang2015noisy}.
In particular, it was shown in \cite{wang2015noisy} that the amount of tolerable stochastic noise could even be substantially larger than the signal in both deterministic and semi-random models.

Nevertheless, the above-mentioned results do not rule out cases when the subgraph of each subspace is not well connected. 
For instance, an empty graph trivially obeys SEP. 
As a less trivial example, if we connect points in each subspace in disjoint pairs, then the degree of every node will be non-zero, yet the graph does not reveal much information for clustering. It is not hard to construct a problem such that Lasso-SSC will output exactly this. 
For the original noiseless SSC, the problem becomes trickier since the solution is more constrained.
In \cite{graph-connectivity} it was shown that when subspace dimension is no larger than 3, SSC outputs block-wise connected similarity graph under very mild conditions;
however, the graph connectivity is easily broken when subspace dimension exceeds 3.
Though a simple post-processing step was remarked in \cite[Footnote 6 in Section 5]{ssc} to alleviate the graph connectivity issue on noiseless data,
it is unclear how to extend their method when data are corrupted by noise.

Among other subspace clustering methods, \cite{park2014greedy} and \cite{heckel2013robust} are the only two papers that provide provable exact clustering guarantees for problems beyond independent subspaces (for which LRR provably gives dense graphs \cite{wang2013lrssc}). Their results however rely critically on the semi-random model assumption. For instance, \cite{heckel2013robust} uses the connectivity of a random k-nearest neighbor graph on a sphere to facilitate an argument for clustering consistency.
In addition, these approaches do not easily generalize to SSC even under the semi-random model since the solution of SSC is considerably harder to characterize.
In contrast, our results are much simpler and work generically without any probabilistic assumptions.
	
	\begin{table}[t]
		\centering
				\caption{Summary of existing theoretical guarantees.  (*) denotes results from this paper.}\label{tab:reviews}
				\vskip 0.1in
		\scalebox{0.85}{
		\begin{tabular}{|ll|c|c|}
			\hline	Algorithm					&& SEP & Exact clustering \\
			\hline LRR &\cite{liu2010robust} &  A-2-a &  A-2-a\\
			\rowcolor{LightCyan}		SSC &\cite{ssc}  & B-2-a  & - \\
			\rowcolor{LightCyan}
			SSC &\cite{sc-geometric} & C-\{1,2\}-a &  - \\
			\rowcolor{LightCyan}
			Noisy SSC &\cite{wang2015noisy} &  C-\{1,2\}-\{a,b,c\} & - \\
			\rowcolor{LightCyan}
			Robust SSC &\cite{robust-ssc} & C-1-\{a,b\} &  -\\
			\rowcolor{LightCyan}
			LRSSC &\cite{wang2013lrssc}  & C-\{1,2\}-a & A-\{1,2\}-a \\
			Thresh. SC &\cite{heckel2013subspace} &  C-1-a & -\\
			Robust TSC &\cite{heckel2013robust} & C-1-\{a,b\} & C-1-\{a,b\}\\
			Greedy SC &\cite{park2014greedy} & C-1-a & C-1-a\\	
			\rowcolor{LightCyan} \textbf{SSC} &(*) & \textbf{C-\{1,2\}-\{a,b,c\}} & \textbf{C-\{1,2\}-\{a,b,c\}}\\
			\hline
		\end{tabular}
		}
		
\end{table}

Lastly, there is a long line of research on ``projective clustering" in the theoretical computer science literature \cite{streaming-pc,coreset-pc}.
Unlike subspace clustering that posits an approximate union-of-subspace model, projective clustering makes no assumption on the data points and is completely agnostic.
The algorithms \cite{streaming-pc,coreset-pc} are typically based on random projection and core-set type techniques,
which are exponential in number of subspaces and/or subspace dimension.
On the other hand, SSC based algorithms are strongly polynomial time in all model parameters.



\section{CLUSTERING CONSISTENT SSC}

In this section, we present and analyze variants of SSC algorithms that outputs consistent clustering with high probability.
As a warm-up exercise, we first consider the case when data are noiseless
and formally establish success conditions for a simple post-processing procedure remarked in \cite{ssc}.
We then move on to our main result in Sec.~\ref{sec:noisy-ssc},
a robustified version of clustering consistent SSC that enjoys perfect clustering condition on data perturbed by a small amount of adversarial noise.
Finally, we construct a counter-example, which shows that our success condition cannot be significantly improved under the adversarial noise model.

\subsection{The noiseless case}\label{sec:noiseless}
%

We first review the procedure of vanilla noiseless Sparse Subspace Clustering (SSC, \cite{ssc,sc-geometric}).
The first step is to solve the following $\ell_1$ optimization problem for each data point $\vct x_i$ in the input matrix $\mat X$:
\begin{eqnarray}
\min_{\vct c_i\in\mathbb R^N}\|\vct c_i\|_1,\label{eq_exact_ssc}\quad
s.t.\;\;\vct x_i=\mat X\vct c_i, \vct c_{ii} = 0.
\end{eqnarray}
Afterwards, a similarity graph $\mat C\in\mathbb R^{N\times N}$ is constructed as $\mat C_{ij}=|[\vct c_i^*]_j|+|[\vct c_j^*]_i|$,
where $\{\vct c_i^*\}_{i=1}^N$ are optimal solutions to Eq. (\ref{eq_exact_ssc}).
Finally, spectral clustering algorithms (e.g.,~\cite{ng2002spectral}) are applied on the similarity graph $\mat C$
to cluster the $N$ data points into $L$ clusters as desired.
Much work has shown that the similarity graph $\mat C$ satisfies SEP under various data and noise regimes \cite{ssc,sc-geometric,wang2015noisy,robust-ssc}.
However, 
as we remarked earlier, SEP alone does not guarantee perfect clustering because the obtained similarity graph $\mat C$ could be poorly connected \cite{graph-connectivity}.
In fact, little is known provably in terms of the final clustering result albeit the practical success of SSC.

We now analyze a simple post-processing procedure of the SSC algorithm (pseudocode displayed in Algorithm \ref{alg_exact_ssc}), which was briefly remarked in \cite{ssc}.
We formally establish that with the additional post-processing step the algorithm achieves consistent clustering under mild ``general-position" conditions.
This simple observation completes previous theoretical analysis of SSC by bridging the gap between SEP and clustering consistency.

\begin{algorithm}[t]
\caption{Clustering consistent noiseless SSC}
\begin{algorithmic}[1]
\State \textbf{Input}: the noiseless data matrix $\mat X$.
\State \textbf{Initialization}: Normalize each column of $\mat X$ so that it has unit two norm.
\State \textbf{Sparse subspace clustering}: Solve the optimization problem in Eq. (\ref{eq_exact_ssc}) for each data point and obtain the similarity matrix $\mat C\in\mathbb R^{N\times N}$.
Define an undirected graph $G=(V,E)$ with $N$ nodes and $(i,j)\in E$ if and only if $\mat C_{ij}>0$.
\State \textbf{Subspace recovery}: For each connected component $G_r=(V_r,E_r)\subseteq G$, compute $\hat{\mathcal S}_{(r)}=\mathrm{Range}(\mat X_{V_r})$ using any convenient linear algebraic method.
Let $\{\hat{\mathcal S}^{(\ell)}\}_{\ell=1}^L$ be the $L$ unique subspaces in $\{\hat{\mathcal S}_{(r)}\}_r$.
\State \textbf{Final clustering}: for each connected component $V_r$ with $\hat{\mathcal S}_{(r)}=\hat{\mathcal S}^{(\ell)}$, set $\hat z_i=\ell$ for all points in $V_r$.
\State \textbf{Output}: cluster assignments $\{\hat z_i\}_{i=1}^N$ and recovered subspaces $\{\hat{\mathcal S}^{(\ell)}\}_{\ell=1}^L$.
\end{algorithmic}
\label{alg_exact_ssc}
\end{algorithm}

The general position condition is formally defined in Definition \ref{defn_gp}, which concerns the distribution of data points within a single subspace.
Intuitively, it requires that no subspace contains data points that are in ``degenerate'' positions.
Similar assumptions were made for the analysis of some algebraic subspace clustering algorithms such as GPCA \cite{gpca}.
The generally positioned data assumption is very mild and is almost always satisfied in practice.
For example, it is satisfied almost surely if data points are i.i.d. generated from any continuous underlying distribution.


\begin{defn}[General position]
Fix $\ell\in\{1,\cdots,L\}$. We say $\mat X^{(\ell)}$ is in \emph{general position} if for all $k\leq d_\ell$, any subset of $k$ data points (columns) in $\mat X^{(\ell)}$ are linearly independent.
We say $\mat X$ is in \emph{general position} if $\mat X^{(\ell)}$ is in general position for all $\ell=1,\cdots,L$.
\label{defn_gp}
\end{defn}

With the self-expressiveness property and the additional assumption that the data matrix $\mat X$ is in general position,
Theorem \ref{thm_main} proves that both the clustering assignments $\{\hat z_i\}_{i=1}^N$
and the recovered subspaces $\{\hat{\mathcal S}^{(\ell)}\}_{\ell=1}^L$ produced by Algorithm \ref{alg_exact_ssc}
are consistent with the ground truth up to permutations.

\begin{thm}[SSC clustering success condition]
Assume $\mat X$ is in general position and no two underlying subspaces are identical.
Let $\{\hat z_i\}_{i=1}^N$ and $\{\hat{\mathcal S}^{(\ell)}\})_{\ell=1}^L$ be the output of Algorithm \ref{alg_exact_ssc}.
If the similarity graph $\mat C$ satisfies the self-expressiveness property as defined in Sec.~\ref{subsec:notation},
then there exists a permutation $\pi$ on $[L]$ such that
$\pi(\hat z_i)=z_i$ and $\hat{\mathcal S}^{(\ell)}=\mathcal S^{(\pi(\ell))}$
for all $i=1,\cdots,N$ and $\ell=1,\cdots,L$.
\label{thm_main}
\end{thm}

The correctness of Theorem \ref{thm_main} is quite straightforward and hence we defer its complete proof to Appendix \ref{appsec:proof_noiseless}.
We also make some comments on the general identifiability and the potential application of $\ell_0$ optimization on union-of-subspace structured data.
As these remarks are only loosely connected to our main results, we state them in Appendix \ref{appsec:identifiability}.
Finally we remark that Algorithm \ref{alg_exact_ssc} only works when the input data are not corrupted by noise.
A non-trivial robust extension is provided in the next section.

\subsection{The noisy case}\label{sec:noisy-ssc}

\begin{algorithm}[t]
\caption{Clustering consistent noisy SSC}
\begin{algorithmic}[1]
\State \textbf{Input}: noisy input matrix $\mat Y$, number of subspaces $L$, intrinsic dimension $d$ and tuning parameter $\lambda$.
\State \textbf{Initialization}: Normalize each column of $\mat X$ so that it has unit two norm.
\State \textbf{Noisy SSC}: Solve the optimization problem in Eq. (\ref{eq_noisy_ssc}) with parameter $\lambda$ for each data point and obtain the similarity matrix $\mat C\in\mathbb R^{N\times N}$.
Define an undirected graph $G=(V,E)$ with $N$ nodes and $(i,j)\in E$ if and only if $\mat C_{ij}>0$.
\State \textbf{Subspace recovery}: For each connected component $G_r=(V_r,E_r)\subseteq G$ with $|V_r|\geq d$,
randomly pick $V_{r,d}\subseteq V_r$ containing exactly $d$ points in $V_r$ and compute $\hat{\mathcal S}_{(r)} = \range(\mat X_{V_{r,d}})$.
\State \textbf{Subspace merging}: Compute the angular distance $d(\hat{\mathcal S}_{(r)}, \hat{\mathcal S}_{(r')})$ as in Eq. (\ref{eq_angular_dist}) for each pair $(r,r')$.
Merge subspaces via single linkage clustering with respect to $d(\cdot,\cdot)$, until there are exactly $L$ subspaces.
\State \textbf{Output}: cluster assignment $\{\hat z_i\}_{i=1}^N$, with $\hat z_i=\hat z_j$ if and only if data points $i$ and $j$ are in the same merged subspace.
\end{algorithmic}
\label{alg_noisy_ssc}
\end{algorithm}

In this section we adopt a noisy input model $\mat Y=\mat X+\mat E$ where $\mat X$ is the noiseless design matrix and $\mat Y$ is the noisy input that is observed.
The noise matrix $\mat E=(\vct\varepsilon_1,\cdots,\vct\varepsilon_N)$ is assumed to be deterministic with $\|\vct\varepsilon_i\|_2\leq\xi$ for every $i=1,\cdots, N$
and some noise magnitude parameter $\xi > 0$.
For noisy inputs $\mat Y$ a Lasso formulation as in Eq. (\ref{eq_noisy_ssc}) is employed for every data point $\vct y_i$.
Choices of the tuning parameter $\lambda$ and
SEP success conditions for Eq. (\ref{eq_noisy_ssc}) have been comprehensively characterized in \cite{wang2015noisy} and \cite{robust-ssc}.
\begin{eqnarray}
\min_{\vct c_i\in\mathbb R^N}& &\frac{1}{2}\|\vct y_i-\mat Y\vct c_i\|_2^2 + \lambda\|\vct c_i\|_1,\label{eq_noisy_ssc}\\
s.t.&& \vct c_{ii} = 0\nonumber.
\end{eqnarray}

We first propose a variant of noisy subspace clustering algorithm (pseudocode listed in Algorithm \ref{alg_noisy_ssc})
that resembles Algorithm \ref{alg_exact_ssc} for the noiseless setting.
For simplicity we assume all underlying subspaces share the same intrinsic dimension $d$ which is known a priori.
The key difference between Algorithm \ref{alg_exact_ssc} and \ref{alg_noisy_ssc}
is that we can no longer unambiguously identify $L$ unique subspaces due to the data noise.
Instead, we employ a single linkage clustering procedure that merges the estimated subspaces that are close with respect to the ``angular distance" measure
between two subspaces, which is defined as
\begin{equation}
d(\mathcal S,\mathcal S') := \|\sin \Phi(\mathcal S,\mathcal S')\|_F^2 = \sum_{i=1}^d{\sin^2\phi_i(\mathcal S,\mathcal S')},
\label{eq_angular_dist}
\end{equation}
where $\{\phi_i(\mathcal S,\mathcal S')\}_{i=1}^d$ are canonical angles between two $d$-dimensional subspace $\mathcal S$ and
$\mathcal S'$.
The angular distance is closely related to the concept of \emph{subspace affinity} defined in \cite{sc-geometric,wang2015noisy}.
In fact, one can show that $d(\mathcal S,\mathcal S') = d-\aff(\mathcal S,\mathcal S')^2$ when both $\mathcal S$ and $\mathcal S'$
are $d$-dimensional subspaces.

In the remainder of this section we present a theorem that proves clustering consistency of Algorithm \ref{alg_noisy_ssc}.
Our key assumption is a restricted eigenvalue assumption, which imposes a lower bound on the smallest singular value
of any subset of $d$ data points within an underlying subspace.

\begin{asmp}[Restricted eigenvalue assumption]
Assume there exist constants $\{\sigma_\ell\}_{\ell=1}^L$ such that
for every $\ell=1,\cdots,L$ the following holds:
\begin{equation}
\min_{\mat X_d=(\vct x_1,\cdots,\vct x_d)\subseteq\mat X^{(\ell)}} \sigma_d(\mat X_d) \geq \sigma_\ell > 0,
\label{eq_re}
\end{equation}
where $\mat X_d$ is taken over all subsets of $d$ data points in the $\ell$th subspace and $\sigma_d(\cdot)$
denotes the $d$th singular value of an $n\times d$ matrix.
\label{asmp_re}
\end{asmp}
Note that Assumption \ref{asmp_re} can be thought of as a robustified version of the ``general position'' assumption in the noiseless case. It requires $\mat X$ to be not only in general position, but also in general position with a spectral margin that is at least $\sigma_\ell$.
In \cite{ssc} a slightly weaker version of the presented assumption was adopted for the analysis of sparse subspace clustering.
We remark further on the related work of restricted eigenvalue assumption at the end of this section.

We continue to introduce the concept of \emph{inradius},
which characterizes the distribution of data points within each subspace and is previously proposed to analyze the SEP success conditions of sparse subspace clustering
\cite{sc-geometric,wang2015noisy}.

\begin{defn}[Inradius, \cite{sc-geometric,wang2015noisy}]
Fix $\ell\in\{1,\cdots,L\}$.
Let $r(\mathcal Q)$ denote the radius of the largest ball inscribed in a convex body $\mathcal Q$.
The inradius $\rho_\ell$ is defined as
\begin{multline}
\rho_\ell = \min_{1\leq i\leq N_\ell}\rho_\ell^{-i} = \min_{1\leq i\leq N_\ell}r(\conv(\pm\vct x_1^{(\ell)}, \cdots, \pm\vct x_{i-1}^{(\ell)},\\
 \pm\vct x_{i+1}^{(\ell)}, \pm\vct x_{N_\ell}^{(\ell)})),
\end{multline}
where $\conv(\cdot)$ denotes the convex hull of a given point set.
\end{defn}

Note that the inradius $\rho_\ell$ is strictly between 0 and 1.
The larger $\rho_\ell$ is, the more uniform data points are distributed in the $\ell$th cluster.
With the restricted eigenvalue assumption and definition of inradius,
we are now ready to present the main theorem of this section
which shows that Algorithm \ref{alg_noisy_ssc} returns consistent clustering when some
conditions on the design matrix, the noise level and range of parameters are met.

\begin{thm}
Assume Assumption \ref{asmp_re} holds and furthermore,
for all $\ell,\ell'\in\{1,\cdots,L\}$, $\ell\neq\ell'$, the following holds:
\begin{eqnarray}
d(\mathcal S^{(\ell)}, \mathcal S^{(\ell')}) &>& \frac{8d\xi^2}{\min_{1\leq t\leq L}\sigma_t^2};\label{eq_subspace_sep}\\
\xi &<& \min\left\{1,\frac{\rho_\ell^2\sigma_\ell}{16(1+\rho_\ell)}\right\}.\label{eq_xi_ub}
\end{eqnarray}
Assume also that the self-expressiveness property holds for the similarity matrix $\mat C$ constructed by Algorithm \ref{alg_noisy_ssc}.
If algorithms parameter $\lambda$ satisfies
\begin{equation}
2\xi(1+\xi)^2(1+1/\rho_\ell)
< \lambda
< \frac{\rho_\ell\sigma_\ell}{2}
\label{eq_thm_lambda_range}
\end{equation}
for every $\ell\in\{1,\cdots,L\}$, then the clustering $\{\hat z_i\}_{i=1}^N$ output by Algorithm \ref{alg_noisy_ssc} is consistent with the ground-truth clustering $\{z_i\}_{i=1}^N$;
that is, there exists a permutation $\pi$ on $\{1,\cdots,L\}$ such that $\pi(\hat z_i) = z_i$ for every $i=1,\cdots, N$.
\label{thm_noisy_ssc}
\end{thm}

A complete proof of Theorem \ref{thm_noisy_ssc} is given in Section \ref{appsec:proof_noisy}.
Below we make several remarks to highlight the nature and consequences of the theorem.

\paragraph{Remark 1}
Let $(\lambda_{\min},\lambda_{\max})$ be the feasible range of $\lambda$
as shown in Eq. (\ref{eq_thm_lambda_range}) in Theorem \ref{thm_noisy_ssc}.
It can be shown that $\lim_{\xi\to 0}\lambda_{\min}= 0$
and $\lim_{\xi\to 0}\lambda_{\max} = \min_\ell\rho_\ell\sigma_\ell/2 > 0$ as long as $\sigma_\ell > 0$ for all $\ell\in\{1,\cdots, L\}$;
that is, $\mat X$ is in general position.
Therefore, the success condition in Theorem \ref{thm_noisy_ssc} reduces to the one in Theorem \ref{thm_main} on noiseless data
when noise diminishes.

\paragraph{Remark 2}
In \cite{wang2015noisy} another range $(\lambda_{\min}',\lambda_{\max}')$ on $\lambda$ is given
for success conditions of the self-expressiveness property.
One can show that $\lim_{\xi\to 0}\lambda_{\min}' = 0$
and $\lim_{\xi\to 0}\lambda_{\max}' = \min_\ell\rho_\ell > 0$.
Therefore, the feasible range of $\lambda$ for both SEP and Theorem \ref{thm_noisy_ssc} to hold is nonempty,
at least for sufficiently low noise level $\xi$.
In addition, the limiting values of $\lambda_{\max}$ and $\lambda_{\max}'$ differ by a factor of $\sigma_\ell/2$
and the maximum tolerable signal-to-noise ratio on $\xi$ differs too by a similar factor of $O(\sigma_\ell)$,
which suggests the difficulty of consistent clustering as opposed to merely SEP for noisy sparse subspace clustering.
In fact, in Sec.~\ref{sec:discussion_asmpre} we construct a counter-example showing that this dependency on $\sigma_\ell$
cannot be improved under the adversarial noise model.

\paragraph{Remark 3}
Some components of Algorithm \ref{alg_noisy_ssc} can be revised to make the method more robust in practical applications.
For example, instead of randomly picking $d$ points and computing their range,
one could apply robust PCA on all points in the connected component, which is more robust to potential outliers.
In addition, the single linkage clustering step could be replaced by $k$-means clustering,
which is more robust to false connections in practice.

\paragraph{Remark 4}
There has been extensive study of using restricted eigenvalue assumptions in the analysis of Lasso-type problems
\cite{lasso-dantzig,oracle-glasso,fussed-lasso-re,rep-cgd}.
However, in our problem the assumption is used in a very different manner.
In particular, we used the restricted eigenvalue assumption to prove one key lemma (Lemma \ref{lem_support_lb})
that \emph{lower bounds} the support size of the optimal solution to a Lasso problem.
Such results might be of independent interest as a nice contribution to the analysis of Lasso in general.

\subsection{Discussion on Assumption \ref{asmp_re}}\label{sec:discussion_asmpre}

Assumption~\ref{asmp_re} requires a spectral gap for every subset of data points in each subspace. 
This seems a very strong assumption that restricts the maximum tolerable noise magnitude to be very small. 
In this section, we show that this dependency on $\sigma_\ell$ is actually necessary for noisy SSC in the adversarial noise setting, which suggests that our bound in Theorem~\ref{thm_noisy_ssc} is sharp.
\begin{prop}\label{prop_lowerbound}
	There is a subspace clustering problem $\mat X\in\mathbb R^{n\times N}$
	and a noise configuration $\mat E\in\mathbb R^{n\times N}$ obeying adversarial noise level $\xi:=\|\mat E\|_{2,\infty} \leq \frac{\sigma_\ell}{\sqrt{d}}$ for some subspace $\ell$
	and intrinsic dimension $d$,
	such that noiseless SSC is clustering consistent on $\mat X$, but noisy SSC on $\mat Y=\mat X+\mat E$ cannot perform better than random guessing.
\end{prop}
\begin{proof}
	 It suffices to come up with one such example.
	 For the sake of simplicity we take intrinsic dimension $d=2$ with $L=4$ clusters.
	 \footnote{The construction of this counter-example can be easily extended to general $d$ cases, as we remark later.}
	 Consider a 2-dimensional subspace $\cS_1$ in $\mathbb R^n$ with orthogonal basis $\mat U_1\in\mathbb R^{n\times 2}$ and assume there are $4$ data points on the subspace represented by
	 $$
	 \mat X^{(1)}=\mat U_1 \mat Z = \mat U_1\begin{bmatrix}
	 1         &  -1 	   & \epsilon  & \epsilon\\
	 \epsilon  & \epsilon & 1 		   &  -1
	 \end{bmatrix}.
	 $$
	 The minimum singular value for the first two points is $\sigma_\ell = \sqrt{2\epsilon}$. This is also the minimum singular value of any pairs of the given points in the subspace. 
	 By taking $\xi =\epsilon = \sigma_\ell/\sqrt{2}$, we can contaminate the data with $\mat E$ to obtain observation data matrix $\mat Y$ as
	 $$
	 \mat Y^{(1)}=\mat U_1 \mat Z + \mat E = \begin{bmatrix}
	 1  &  -1 	   & 0  & 0\\
	 0  &  0	& 1 		   &  -1
	 \end{bmatrix}.
	 $$
	 Assume there is another subspace $\cS_2 \perp \cS_1$ with the four data points $\mat X^{(2)} = \mat U_2\mat Z$, 
	 and we contaminate them in the same fashion into $\mat Y^{(2)}$. 
	 Noiseless SSC on $\mat X$ is trivially clustering consistent by Theorem~\ref{thm_main}. 
	 Noisy SSC on $\mat Y$ however will construct a graph that has exactly 4 connected components with any $\lambda$ that returns a non-zero solution. These are:
	$$\{1,2\},\{3,4\},\{5,6\},\{7,8\}$$
	Spectral clustering algorithms that tries to partition the graph into $2$ parts will not be able to work better than random labeling. Similarly, Algorithm~2 will also fail because the subspace spanned by the noisy data points in each connected components are mutually orthogonal, and no ``merging'' procedure will be able to consistently recover the original subspace assignments.
\end{proof}

\begin{figure*}[t]
	\centering
	\includegraphics[width=0.48\textwidth]{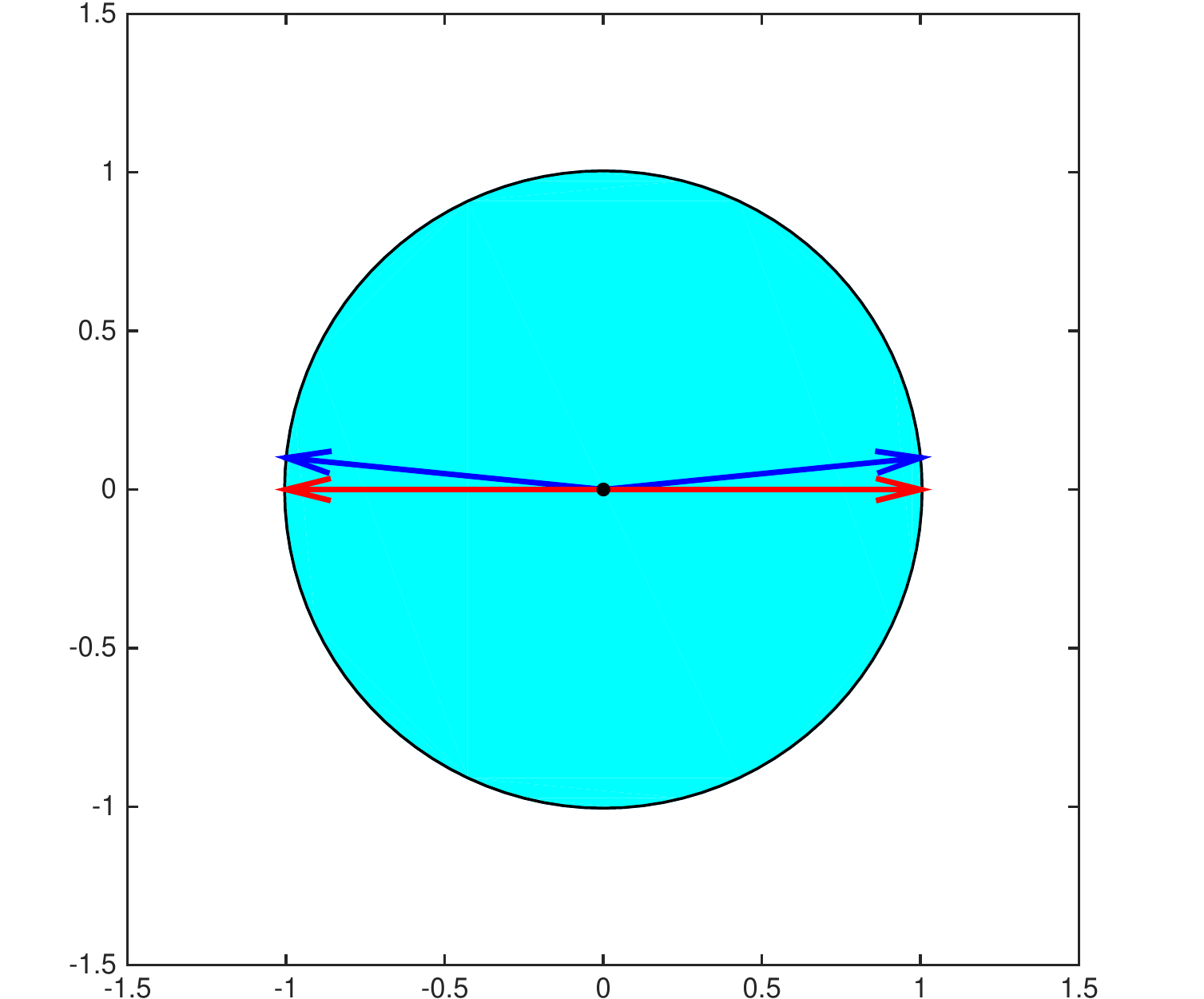}
	\includegraphics[width=0.48\textwidth]{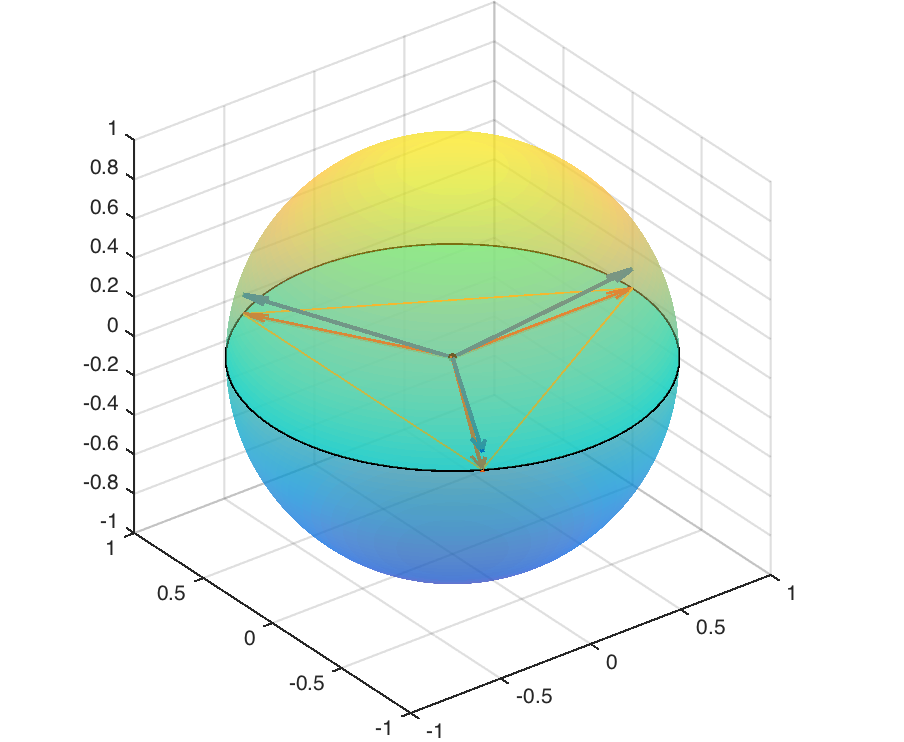}
	\caption{An illustration of counter-examples constructed in Proposition~\ref{prop_lowerbound}. \textbf{Left:} a 2D example. \textbf{Right:} a 3D example. The arrows in blue represent the noiseless data in general position. The arrows in red illustrate how a small perturbation of size $\sigma_{\ell}/\sqrt{d}$ can potentially break the general position assumption.}
	\label{fig_counterexample}
\end{figure*}

The high level idea of this example is that $\sigma_\ell$ measures how close the data points in subspace $\ell$ are from violating the general position assumption and therefore with an arbitrary perturbation of magnitude $\sigma_\ell$, we can change at least $d$ points to lie in an $(d-1)$-dimensional subspace, which renders the original problem non-identifiable.
\paragraph{Remark 5}
For any intrinsic dimension $d\geq 2$, we can construct a set of $d$ points in general position where one only needs to perturb each data point by $\sigma_\ell/\sqrt{d}$ to made them lie in a $d-1$ dimensional subspace space. 
Fix any orthonormal basis of $\mathbb R^d$ (without loss of generality we work under the standard basis $[\vct e_1,\cdots,\vct e_d]$).
The $d$ points are linear combinations of these basis with coefficients
$$\begin{bmatrix}
\vct \beta_1         &  \vct\beta_2 	   & ...  & \vct \beta_d\\
\sigma_\ell/\sqrt{d}  & \sigma_\ell/\sqrt{d}	& ... 		   &  \sigma_\ell/\sqrt{d}
\end{bmatrix}$$
where we set $\{\vct \beta_i\}$ to be the $d$ vertices of a symmetric simplex in $\R^{d-1}$ with centroid at the origin. Just to give a few examples, in $\R$ this is $\{-1,1\}$ and in $\R^2$ this is $\left\{ \begin{bmatrix} 1\\0
\end{bmatrix}, \begin{bmatrix} -0.5\\\sqrt{3}/2
\end{bmatrix}, \begin{bmatrix} -0.5\\-\sqrt{3}/2
\end{bmatrix} \right\}$.
The construction of such examples is illustrated in Figure~\ref{fig_counterexample}. In general, since all these vectors are orthogonal to $\vct e_d$, and the way they are constructed ensures that the top $d-1$ singular values are all identically $\sqrt{d/(d-1)}$, the minimum singular value will be exactly $\sigma_\ell$ and by adversarial perturbation of size $\sigma_\ell/\sqrt{d}$ on each data point we reduce all points to a $\R^{d-1}$ subspace and hence they are no longer in general position.

\section{SIMULATIONS}

\begin{figure*}[t]
\centering
\includegraphics[width=5cm]{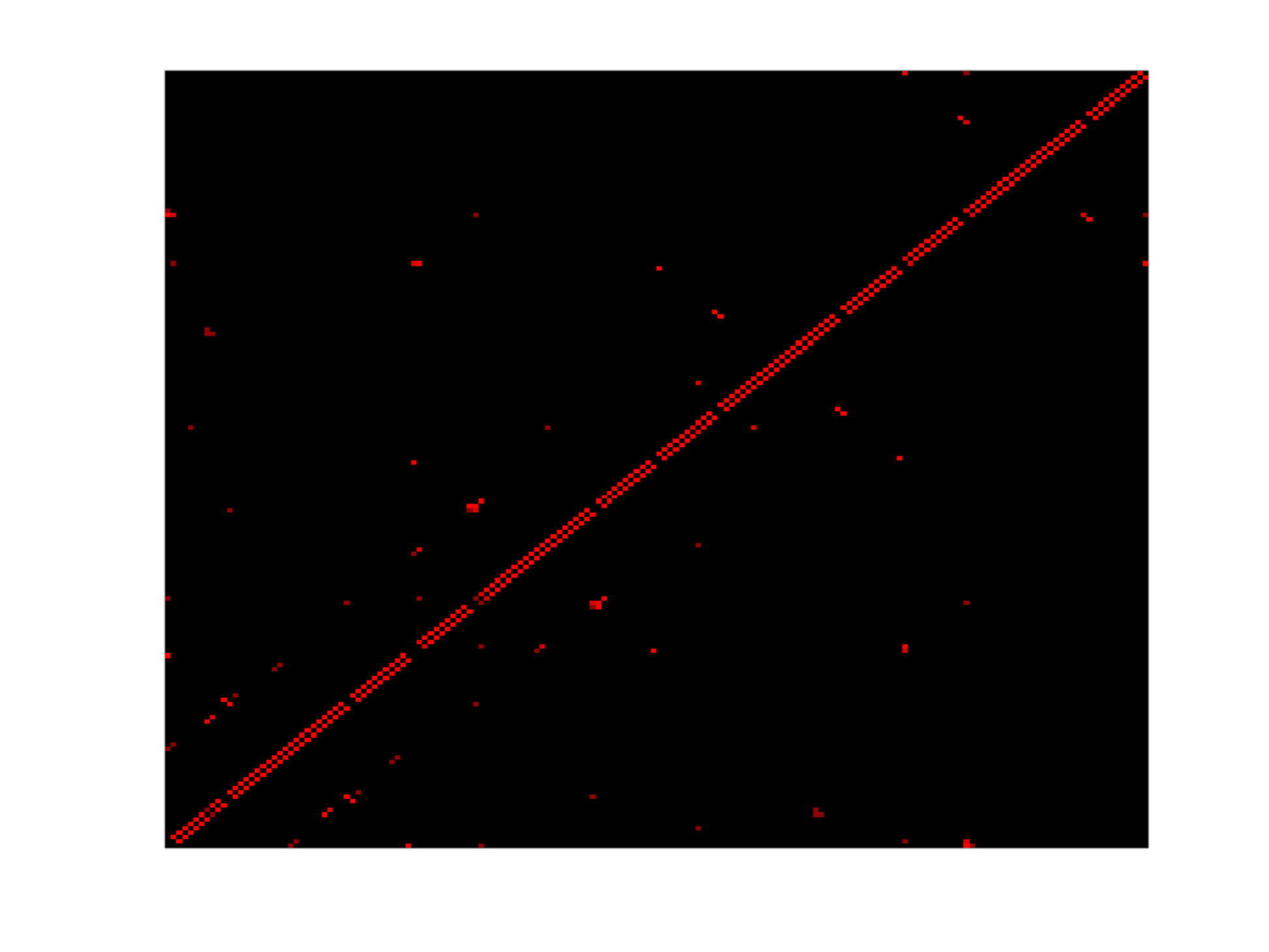}
\includegraphics[width=5cm]{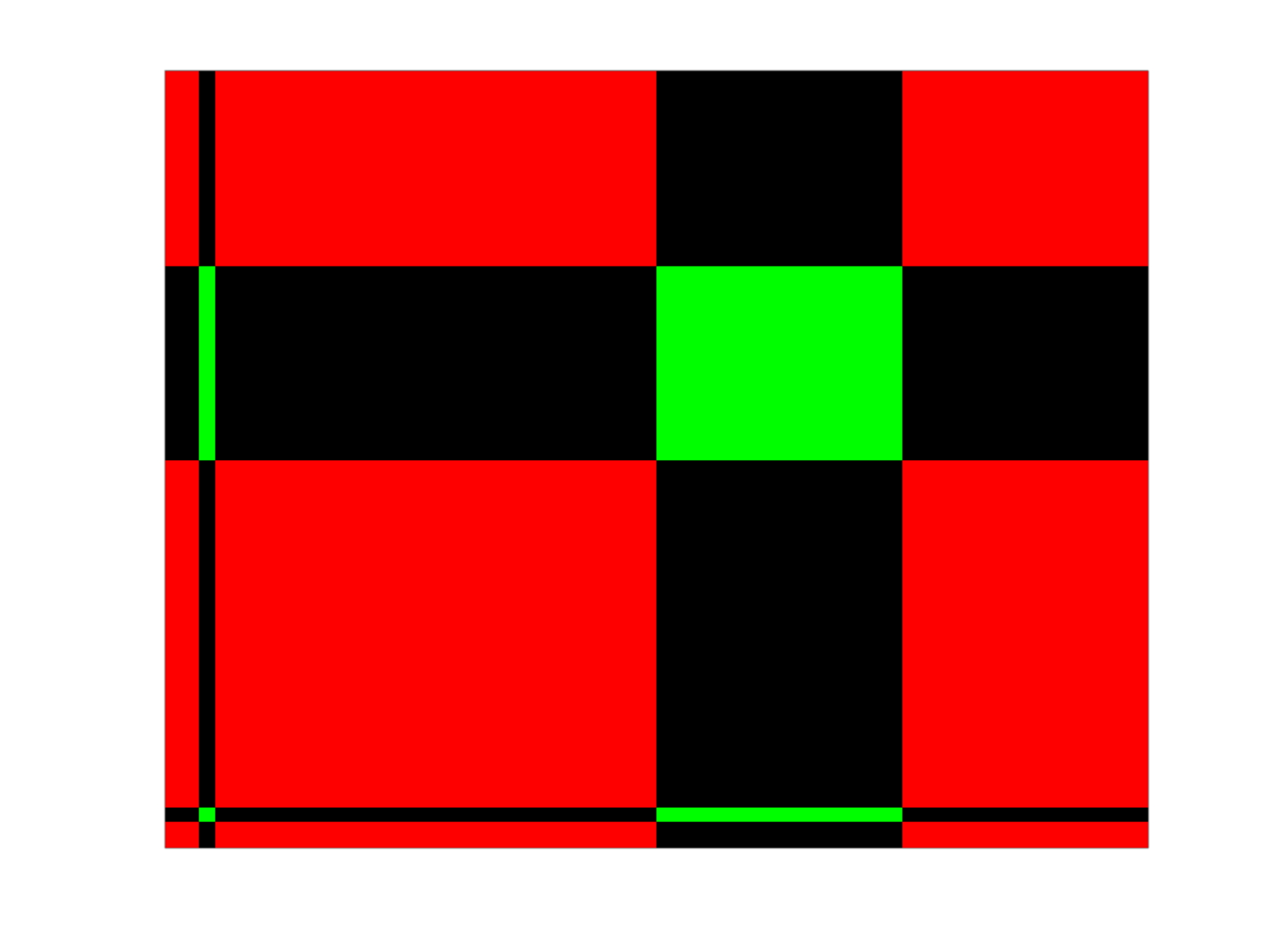}
\includegraphics[width=5cm]{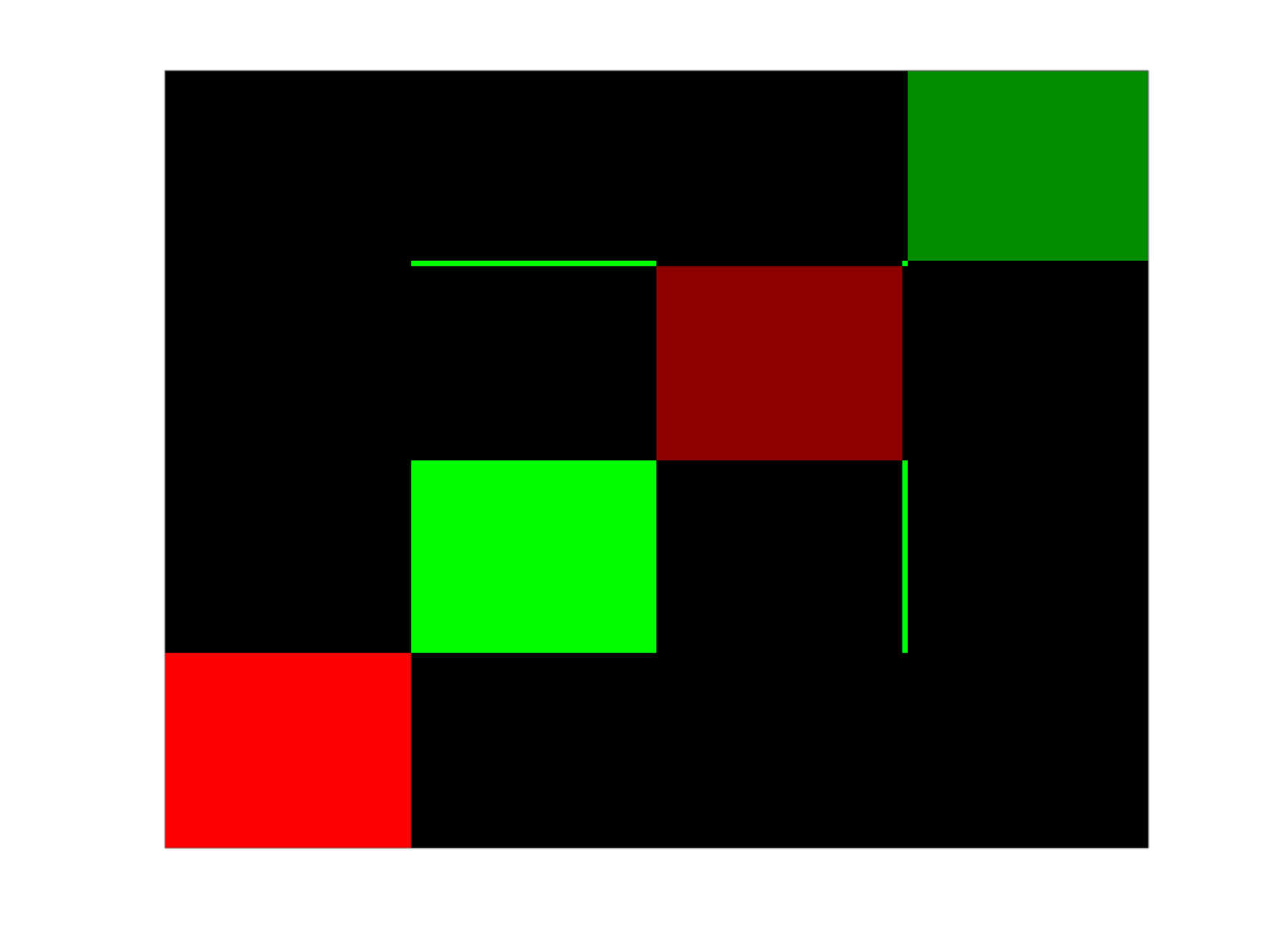}
\includegraphics[width=5cm]{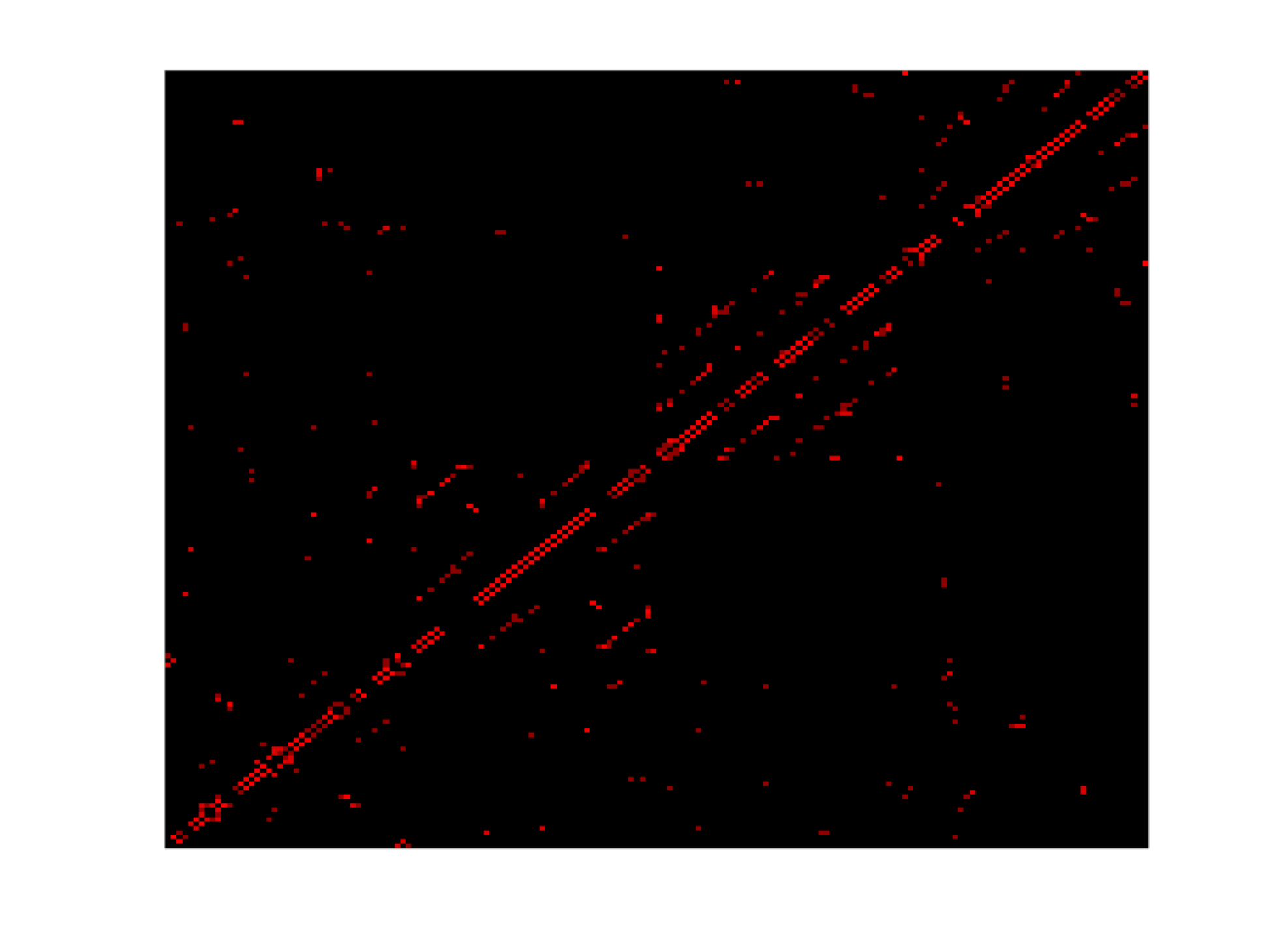}
\includegraphics[width=5cm]{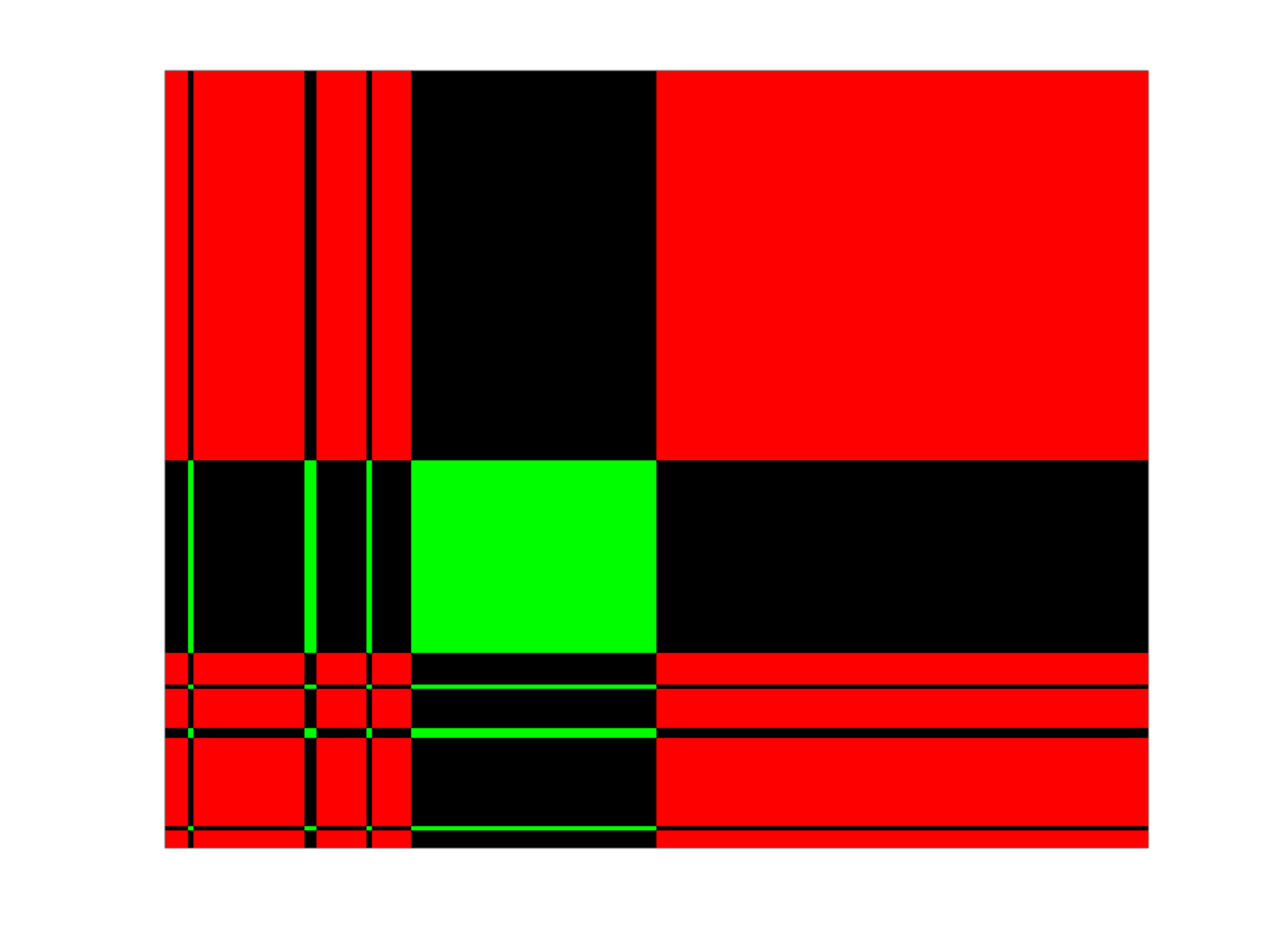}
\includegraphics[width=5cm]{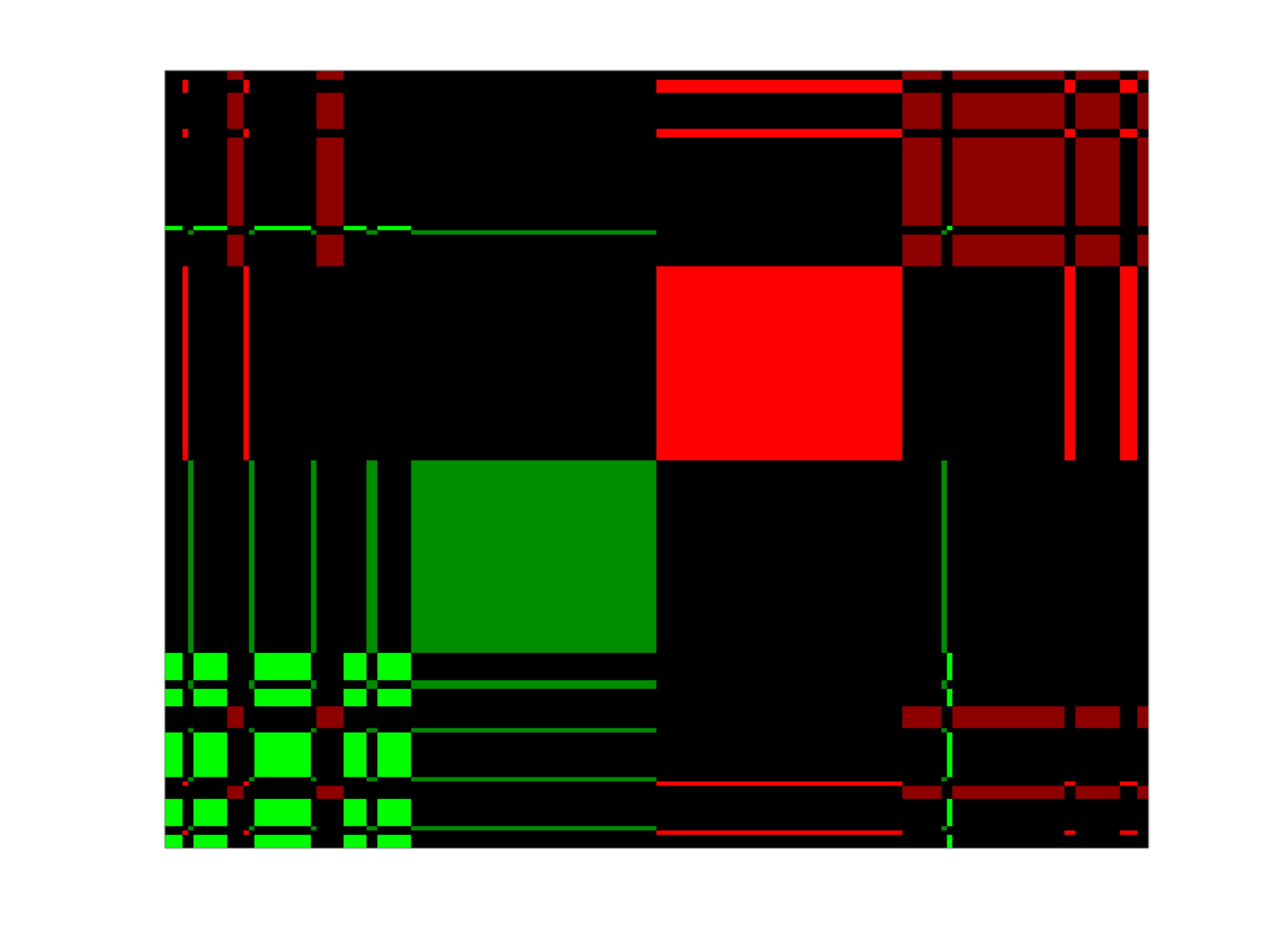}
\caption{Clustering analysis on noiseless (top) and noisy (bottom) data.
Left: similarity matrix produced by Lasso SSC.
Middle: spectral clustering on the similarity matrix, with 2 clusters.
Right: spectral clustering on the similarity matrix, with 4 clusters.}
\label{fig_all}
\end{figure*}

In this section we report simulation results of our proposed algorithms on the example constructed by Nasihatkon and Hartley in \cite{graph-connectivity}.
It was shown in \cite{graph-connectivity} that such an example will result in highly disconnected similarity graphs, and thus poses
a unique challenge for spectral clustering to recover the true clustering of data points.
In particular, consider 4-dimensional subspaces and for each subspace we generate data set $\mathcal A$ consisting of $8m$ data points in $\mathbb R^4$ as follows:
\begin{multline}
\mathcal A = \bigcup_{k=0}^{m-1}\bigcup_{s,s'\in\{\pm 1\}} \{(\cos\theta_k,\sin\theta_k,s\delta,s'\delta),\\
(s\delta,s'\delta,\cos\theta_k,\sin\theta_k)\};\quad \theta_k = k\pi/m,
\end{multline}
where $m\in\mathcal N^*$ and $\delta\in(0,1)$ are parameters for generating the data set.
Finally, the unnormalized observation matrix $\widetilde{\mat X}$ is constructed as
$$
\widetilde{\mat X} = \left[\mat W_1\mat A, \mat W_2\mat A\right],
$$
where $\mat W_1,\mat W_2\in\mathbb R^{n\times 4},n > 4$ are different linear operators that map a 4-dimensional vector to an $n$-dimensional ambient space.
Finally, the input matrix $\mat X$ is obtained by normalizing $\widetilde{\mat X}$ so that each column has unit $\ell_2$ norm and then adding Gaussian white noise with entry-wise variance $\sigma^2/n$.

\begin{table}[t]
\centering
\caption{Relative Violation (Rel. Vio.) of SEP, clustering accuracy without post-processing (Acc. 1) and clustering accuracy with post-processing
(Acc. 2) for Lasso SSC on noiseless and noisy data.}
\vskip 0.1in
\begin{tabular}{lccc}
\hline
& Rel. Vio.& Acc. 1& Acc. 2\\
\hline
Noiseless& .03& .73& .99\\
Noisy& .09& .77& .93\\
\hline
\end{tabular}
\label{tab_result}
\end{table}

Before presenting the simulation results we first make some remarks on the constructed dataset $\mat X$.
By construction, $\mat X$ has two overlaping 4-dimensional subspaces with probability 1, if both $\mat W_1$ and $\mat W_2$ are sampled uniformly from all orthogonal linear mappings from
$\mathbb R^4$ to $\mathbb R^n$.
Furthermore, noiseless data points in each cluster are in general position, provided that $m$ is a prime number.,
In \cite{graph-connectivity} it was shown that SSC tends to cluster data points in each cluster into two disjoint clusters.
Hence, the follow-up spectral clustering step cannot correctly merge the four learnt clusters into two without additional information.

In Figure \ref{fig_all} we plot the similarity graph learnt by Lasso SSC as well as spectral clustering results on both noiseless and noisy data.
The parameters for data generation are set as $n=5$, $m=11$, $\delta=0.2$, $\sigma=0.1$ and Lasso SSC parameter is set as $\lambda=10^{-3}$.
Figure \ref{fig_all} shows that the similarity graph is poorly connected and hence if we try to directly cluster the data points into two clusters (the middle column of the plots),
the spectral clustering algorithm fails completely.
On the other hand, it does a good job in clustering the data points into 4 clusters.
Subsequently, we could apply our proposed post-processing step by first computing the underlying low-dimensional subspace for each cluster and
then merge those subspaces that are close in angular distance.
As a result, near perfect clustering could be achieved on this synthetic dataset, as shown in Table \ref{tab_result}.
We also report the relative violation of SEP property
\footnote{The relative violation of SEP for a similarity graph $\mat C$ is defined as $\sum_{(i,j)\in E}{|\mat C_{ij}|}/\sum_{(i,j)\notin E}{|\mat C|_{ij}}$,
where $(i,j)\in E$ if and only if $\vct x_i$ and $\vct x_j$ belong to the same cluster.}
 in Table \ref{tab_result} to show that the SEP property is very well satisfied and is hence not a contributing factor for the poor performance of vanilla Lasso SSC.
 
\section{CONCLUSION}

In this paper we investigate graph connectivity in noisy sparse subspace clustering.
We propose a robust post-process step of noisy SSC that produces consistent clustering with high probability,
assuming the magnitude of noise is sufficiently small.
Our work is the first step toward noisy SSC with complete clustering guarantees, under the most general fully deterministic data model.
We next remark on several future directions along this line of research,
which could further improve the results presented in this paper.

Perhaps the most important limitation of Theorem \ref{thm_noisy_ssc} is the restricted eigenvalue assumption (Assumption \ref{asmp_re}).
Since it concerns the smallest singular value of the most ill-posed subset of $d$ data points,
we are really requiring the noise magnitude of $\xi$ to be extremely small.
In fact, we believe  $\sigma_\ell$ is exponentially small with respect to the number of data points per subspace,
assuming they are drawn uniformly from the unit low-dimensional sphere.
Although getting a better dependency over $\sigma_\ell$ is impossible under the adversarial noise model 
(as shown in Sec.~\ref{sec:discussion_asmpre}),
we conjecture that the assumption could be relaxed when noise are stochastic such as Gaussian white noise.

Another potential fruitful direction is to relax the requirement that the support of sparse regression for every data point
consists of at least $d$ other data points.
With less than $(d+1)$ data points in a connected component we can no longer approximately estimate the intrinsic low-dimensional subspace;
however, we might still be able to obtain some leading directions of the underlying subspace,
which could provide valuable information for the subspace merging step.
In fact, Soltanokoltabi et al. proved lower bounds on support size in robust subspace clustering
under the semi-random model setting \cite{robust-ssc}.
Though their bound is not as tight as $\Omega(d)$, it may benifit from some additional post-processing step
that attempts to merge over-clustered subspaces together.

\bibliographystyle{IEEE}
\bibliography{sc-note}

\begin{thebibliography}{10}

\bibitem{lasso-dantzig}
P.~J. Bickel, Y.~Ritov, and A.~B. Tsybakov.
\newblock Simultaneous analysis of lasso and dantzig selector.
\newblock {\em The Annals of Statistics}, 37(4):1705--1732, 2009.

\bibitem{fussed-lasso-re}
Y.~Chen and A.~Dalayan.
\newblock Fused sparsity and robust estimation for linear models with unknown
  variance.
\newblock In {\em NIPS}, 2012.

\bibitem{chen2014clustering}
Y.~Chen, A.~Jalali, S.~Sanghavi, and H.~Xu.
\newblock Clustering partially observed graphs via convex optimization.
\newblock {\em The Journal of Machine Learning Research}, 15(1):2213--2238,
  2014.

\bibitem{ssc}
E.~Elhamifar and R.~Vidal.
\newblock Sparse subspace clustering: Algorithm, theory and applications.
\newblock {\em IEEE Transactions on Pattern Analysis and Machine Intelligence},
  35(11):2765--2781, 2013.

\bibitem{eriksson2012high}
B.~Eriksson, L.~Balzano, and R.~Nowak.
\newblock High-rank matrix completion.
\newblock In {\em AISTATS}, 2012.

\bibitem{coreset-pc}
D.~Feldman, M.~Schmidt, and C.~Sohler.
\newblock Turning big data into tiny data: constant-size coresets for k-means,
  pca and projective clustering.
\newblock In {\em SODA}, 2013.

\bibitem{heckel2013robust}
R.~Heckel and H.~B{\"o}lcskei.
\newblock Robust subspace clustering via thresholding.
\newblock {\em arXiv:1307.4891}, 2013.

\bibitem{heckel2013subspace}
R.~Heckel and H.~B{\"o}lcskei.
\newblock Subspace clustering via thresholding and spectral clustering.
\newblock In {\em ICASSP}, 2013.

\bibitem{ho2003clustering}
J.~Ho, M.-H. Yang, J.~Lim, K.-C. Lee, and D.~Kriegman.
\newblock Clustering appearances of objects under varying illumination
  conditions.
\newblock In {\em CVPR}, 2003.

\bibitem{hong2006multiscale}
W.~Hong, J.~Wright, K.~Huang, and Y.~Ma.
\newblock Multiscale hybrid linear models for lossy image representation.
\newblock {\em IEEE Transactions on Image Processing}, 15(12):3655--3671, 2006.

\bibitem{streaming-pc}
M.~Kerber and S.~Raghvendra.
\newblock Approximation and streaming algorithms for projective clustering via
  random projections.
\newblock {\em arXiv:1407.2063}, 2014.

\bibitem{liu2010robust}
G.~Liu, Z.~Lin, and Y.~Yu.
\newblock Robust subspace segmentation by low-rank representation.
\newblock In {\em ICML}, 2010.

\bibitem{oracle-glasso}
K.~Lounici, M.~Pontil, A.~Tsybakov, and S.~van~de Geer.
\newblock Oracle inequalities and optimal inference under group sparsity.
\newblock {\em The Annals of Statistics}, 39(4):2164--2204, 2011.

\bibitem{mcwilliams2014subspace}
B.~McWilliams and G.~Montana.
\newblock Subspace clustering of high-dimensional data: a predictive approach.
\newblock {\em Data Mining and Knowledge Discovery}, 28(3):736--772, 2014.

\bibitem{graph-connectivity}
B.~Nasihatkon and R.~Hartley.
\newblock Graph connectivity in sparse subspace clustering.
\newblock In {\em CVPR}, 2011.

\bibitem{ng2002spectral}
A.~Ng, M.~Jordan, and Y.~Weiss.
\newblock On spectral clustering: analysis and an algorithm.
\newblock In {\em NIPS}, 2002.

\bibitem{park2014greedy}
D.~Park, C.~Caramanis, and S.~Sanghavi.
\newblock Greedy subspace clustering.
\newblock In {\em NIPS}, 2014.

\bibitem{rep-cgd}
G.~Raskutti, M.~Wainwright, and B.~Yu.
\newblock Restricted eigenvalue properties for correlated gaussian designs.
\newblock {\em Journal of Machine Learning Research}, 11:2241--2259, 2010.

\bibitem{sc-geometric}
M.~Soltanolkotabi, E.~J. Candes, et~al.
\newblock A geometric analysis of subspace clustering with outliers.
\newblock {\em The Annals of Statistics}, 40(4):2195--2238, 2012.

\bibitem{robust-ssc}
M.~Soltanolkotabi, E.~Elhamifar, and E.~Candes.
\newblock Robust subspace clustering.
\newblock {\em The Annals of Statistics}, 42(2):669--699, 2014.

\bibitem{mat-pert-theory}
G.~W. Stewart, J.-g. Sun, and H.~B. Jovanovich.
\newblock {\em Matrix perturbation theory}.
\newblock Academic press New York, 1990.

\bibitem{lassodegree}
R.~J. Tibshirani, J.~Taylor, et~al.
\newblock Degrees of freedom in lasso problems.
\newblock {\em The Annals of Statistics}, 40(2):1198--1232, 2012.

\bibitem{gpca}
R.~Vidal, Y.~Ma, and S.~Sastry.
\newblock Generalized principal component analysis ({GPCA}).
\newblock {\em IEEE Transactions on Pattern Analysis and Machine Intelligence},
  27(12):1945--1959, 2005.

\bibitem{wang2015noisy}
Y.-X. Wang and H.~Xu.
\newblock Noisy sparse subspace clustering.
\newblock {\em arXiv:1309.1233}, 2013.

\bibitem{noisy-ssc}
Y.-X. Wang and H.~Xu.
\newblock Noisy sparse subspace clustering.
\newblock In {\em ICML}, 2013.

\bibitem{wang2013lrssc}
Y.-X. Wang, H.~Xu, and C.~Leng.
\newblock Provable subspace clustering: When {LRR} meets {SSC}.
\newblock In {\em NIPS}, 2013.

\bibitem{yan2006general}
J.~Yan and M.~Pollefeys.
\newblock A general framework for motion segmentation: Independent,
  articulated, rigid, non-rigid, degenerate and non-degenerate.
\newblock In {\em ECCV}, 2006.

\bibitem{zhang2012guess}
A.~Zhang, N.~Fawaz, S.~Ioannidis, and A.~Montanari.
\newblock Guess who rated this movie: Identifying users through subspace
  clustering.
\newblock {\em arXiv:1208.1544}, 2012.

\end{thebibliography}

\newpage
\onecolumn

\begin{appendices}

\section{PROOFS OF THEOREMS FOR NOISELESS SSC}\label{appsec:proof_noiseless}

We prove Theorem \ref{thm_main}, the main theorem for the noiseless clustering consistent SSC algorithm given in Sec.~\ref{sec:noiseless}

\begin{proof}[Proof of Theorem \ref{thm_main}]
Fix a connected component $G_r=(V_r,E_r)\subseteq G$.
By the self-expressiveness property we know that all data points in $V_r$ lie on the same underlying subspace $\mathcal S^{(\ell)}$.
It can be easily shown that if $\mat X^{(\ell)}$ is in general position then $|V_r|\geq d_\ell+1$ because for any $\vct x_i\in\mathcal S^{(\ell)}$,
at least $d_\ell$ other data points in the same subspace are required to perfectly reconstruct $\vct x_i$.
Consequently, we have $\hat{\mathcal S}_{(r)}=\mathcal S^{(\ell)}$ because $V_r$ contains at least $d_\ell$ data points in $\mathcal S^{(\ell)}$ that are linear independent.
On the other hand, due to the self-expressiveness property,
for every $\ell=1,\cdots, L$ there exists a connected component $G_r$ such that $\hat{\mathcal S}_{(r)}=\mathcal S^{(\ell)}$
because otherwise nodes in $\mat X^{(\ell)}$ will have no edges attached, which contradicts Eq. (\ref{eq_exact_ssc}) and the definition of $G$.
As a result, the above argument shows that Algorithm \ref{alg_exact_ssc} achieves perfect subspace recovery;
that is, there exists a permutation $\pi$ on $[L]$ such that $\hat{\mathcal S}^{(\ell)}=\mathcal S^{(\pi(\ell))}$ for all $\ell=1,\cdots,L$.

We next prove that Algorithm \ref{alg_exact_ssc} achieves perfect clustering as well, that is, $\pi(\hat z_i)=z_i$ for every $i=1,\cdots, N$.
Assume by way of contradiction that there exists $i$ such that $\hat z_i=\ell$ and $z_i = \ell'\neq\pi(\ell)$.
Let $G_r=(V_r,E_r)\subseteq G$ be the connected component in $G$ that contains the node corresponding to $\vct x_i$.
Since $\hat z_i=\ell$, by SEP and the above analysis we have $\hat{\mathcal S}_{(r)}=\hat{\mathcal S}^{(\ell)}=\mathcal S^{(\pi(\ell))}$.
On the other hand, because $z_i=\ell'$ and data points in $V_r$ are in general position, we have $\hat{\mathcal S}_{(r)}=\mathcal S^{(\ell')}$.
Hence, $\mathcal S^{(\pi(\ell))}=\mathcal S^{(\ell')}$ with $\ell'\neq\pi(\ell)$, which contradicts the assumption that
no two underlying subspaces are identical.
\end{proof}

\section{DISCUSSION ON IDENTIFIABILITY AND $\ell_0$ FORMULATION OF NOISELESS SUBSPACE CLUSTERING}\label{appsec:identifiability}

\subsection{The identifiability of noiseless subspace clustering}
If we use a more relaxed notion of identifiability, even the ``general position'' assumption could be dropped for consistent clustering.
In Theorem \ref{thm:minimal_truth} we define such a relaxed notion of identifiability for the union-of-subspace structure.
\begin{thm}\label{thm:minimal_truth}
	Any set of $N$ data points in $\R^n$ has a partition that follows a union-of-subspace structure, where points in each subspaces are in general position. We call this partition the \emph{minimal} union-of-subspace structure.
\end{thm}

\begin{proof} 
	Given a finite set $\mathcal{X} \subset \R^n$. We will algorithmically construct a minimal partition. Initialize set $\mathcal{Y} = \mathcal{X}$. Start with $k=1$, do the following repeatedly until it fails, then increment $k$, until $\mathcal{Y}=\emptyset$: find the maximum number of points that lie in a hyperplane of dimension $(k+1)$, assign a new partition for these points and remove these points from $\mathcal{Y}$.
	It is clear that in this way, every partition is a distinct subspace and points in any subspace are in general position.
\end{proof}

 One consequence of Theorem \ref{thm:minimal_truth} is that if SEP holds with respect to any minimal union-of-subspace structure (i.e., a minimal ground truth), then Algorithm~\ref{alg_exact_ssc} will recover the correct ground truth clustering.
 {We remark that SEP does not hold for any finite subset of points in $\mathbb R^n$ if $\ell_1$ regularization is used,
 unless the data satisfy certain separation conditions \cite{sc-geometric}.
 However, in Section \ref{subsec:agnostic-sc} we propose an $\ell_0$ regularization problem which achieves SEP (and hence consistent clustering)
 for any $\mathcal X\subseteq\mathbb R^d$.}

 We note that the minimal union-of-subspace structure may not be unique. An example is that if there is one point in the intersection of two subspaces with equal dimension, then this point can be assigned to either subspaces. Now, suppose the intersection has dimension $k$, there can be at most $k$ points in the intersection, otherwise these points will form a new $k$-dimension subspace and the original structure is no longer minimal.

\subsection{The merit of $\ell_0$-minimization and agnostic subspace clustering}\label{subsec:agnostic-sc}
A byproduct of our result is that it also addresses an interesting question of whether it is advantageous to use $\ell_0$ over $\ell_1$ minimization in subspace clustering, namely
\begin{equation}
\min_{\vct c_i\in\mathbb R^N} \;\|\vct c_i\|_0,\;\;
s.t.\;\;\vct x_i=\mat X\vct c_i, \vct c_{ii} = 0.\label{eq_l0_ssc}
\end{equation}
If one poses this question to a compressive sensing researcher, the answer will most likely be yes, since $\ell_0$ minimization is the original problem of interest and empirical evidence suggests that using iterative re-weighted $\ell_1$ scheme to approximate $\ell_0$ solutions often improves the quality of signal recovery. On the other hand, a statistician is most likely to answer just the opposite because $\ell_1$ shrinkage would often significantly reduce the variance at the cost of a small amount of bias. A formal treatment of the latter intuition suggests that $\ell_1$ regularized regression has strictly less ``effective-degree-of-freedom'' than the ``$\ell_0$ best-subset selection''  \cite{lassodegree}, therefore generalizes better.

How about subspace clustering? Unlike $\ell_1$ solution that is unique almost everywhere, $\ell_0$ solutions will not be unique and it is easy to construct a largely disconnected graph based on optimal $\ell_0$ solutions. Using the new observation that we do not actually need graph connectivity, we are able to establish that  $\ell_0$ minimization for SSC is indeed the ultimate answer for noiseless subspace clustering.
\begin{thm}\label{thm:l0}
Given any $N$ points in $\R^d$, any solutions to the $\ell_0$-variant of Algorithm~\ref{alg_exact_ssc} will partition the points into a minimal union-of-subspace structure.
\end{thm}
\begin{proof} 
	Define a \emph{minimal} subspace with respect to point $\vct x_i$ in a set $\{\vct x_i\}_{i=1}^N$ to be the span of any points that minimizes \eqref{eq_l0_ssc} for $i$. Since the ordering of how data points are used does not matter in Algorithm~\ref{alg_exact_ssc}, we can sort the points into an ascending order with respect to the dimensionality. Now the merging procedure of these subspaces into a unique set of subspaces is exactly the same as the construction in the proof of Theorem~\ref{thm:minimal_truth}. Therefore, all solutions of the $\ell_0$ SSC are going to be the correct partition.
\end{proof}

With slightly more effort, it can be shown that the converse is also true. Therefore, the set of solutions of $\ell_0$-SSC completely characterizes the set of minimal union-of-subspace structure for any set of points in $\R^d$. In contrast, $\ell_1$-SSC requires additional separation condition to work. That said, it may well be the case in practice that $\ell_1$-SSC works better for the noisy subspace clustering in the low signal-to-noise ratio regime. It will be an interesting direction to explore how iterative reweighted $\ell_1$ minimizations and local optimization for $\ell_p$-norm ($0<p<1$) work in subspace clustering applications.

\section{PROOFS OF THEOREMS FOR NOISY SSC}\label{appsec:proof_noisy}

The purpose of this section is to present a complete proof to Theorem \ref{thm_noisy_ssc},
our main result concerning clustering consistent Lasso SSC on noisy data.
We first present and prove two technical propositions that will be used later.
\begin{prop}
Let $\vct u$ be an arbitrary vector in $\mathcal S^{(\ell)}$ with $\|\vct u\|_2=1$.
Then $\max_{1\leq i\leq N_\ell,i\neq i^*}|\langle\vct u,\vct x_i^{(\ell)}\rangle| \geq \rho_\ell^{-i^*}$
for every $i^*=1,\cdots,N_\ell$.
\label{prop_innerprod_lb}
\end{prop}
\begin{proof}
For notational simplicity let $\mat X_{-i^*}^{(\ell)} = (\vct x_1^{(\ell)}, \cdots, \vct x_{i^*-1}^{(\ell)},\vct x_{i^*+1}^{(\ell)},
\cdots, \vct x_{N_\ell}^{(\ell)})$
and $\mathcal Q_{-i^*}^{(\ell)} = \conv(\pm\mat X_{-i^*}^{(\ell)})$.
The objective of Proposition \ref{prop_innerprod_lb} is to lower bound $\|\mat X_{-i^*}^{(\ell)^\top}\vct u\|_{\infty}$
for any $\vct u\in\mathcal S^{(\ell)}$ with $\|\vct u\|_2 = 1$.
By definition of the dual norm, $\|\mat X_{-i^*}^{(\ell)^\top}\vct u\|_{\infty}$ is equal to the objective of the following optimization problem
\begin{equation}
\max_{\vct c\in\mathbb R^{N_\ell-1}}\langle\vct u, \mat X_{-i^*}^{(\ell)}\vct c\rangle\quad
s.t.\;\;\|\vct c\|_1 = 1.
\label{eq_infty_dualnorm}
\end{equation}
To obtain a lower bound on the objective of Eq. (\ref{eq_infty_dualnorm}),
note that $\rho_\ell^{-i^*}$ is the radius of the largest ball inscribed in $\mathcal Q_{-i^*}^{(\ell)}$
and hence $\rho_\ell^{-i^*}\vct u\in\mathcal Q_{-i^*}^{(\ell)}$.
Consequently, $\rho_\ell^{-i^*}\vct u$ can be written as a convex combination of (signed) columns in $\mat X_{-i^*}^{(\ell)}$,
that is, there exists $\vct c\in\mathbb R^{N_\ell-1}$ with $\|\vct c\|_1=1$ such that
$\mat X_{-i^*}^{(\ell)}\vct c = \rho_\ell^{-i^*}\vct u$.
Plugging the expression into Eq. (\ref{eq_infty_dualnorm}) we obtain
$$
\|\mat X_{-i^*}^{(\ell)^\top}\vct u\|_\infty \geq \langle\vct u, \rho_\ell^{-i^*}\vct u\rangle = \rho_\ell^{-i^*}.
$$
\end{proof}

\begin{prop}
Let $\mat A=(\vct a_1,\cdots, \vct a_m)$ be an arbitrary matrix with at least $m$ rows. Then
$\|\vct a_i - \mathcal P_{\range(\vct a_{-i})}(\vct a_i)\|_2 \geq \sigma_m(\mat A)$,
where $\vct a_{-i}$ denotes all columns in $\mat A$ except $\vct a_i$.
\label{prop_perp_norm}
\end{prop}
\begin{proof}
Denote $\vct a_i^\perp$ as $\vct a_i^\perp = \vct a_i - \mathcal P_{\range(\vct a_{-i})}(\vct a_i)$.
By definition, $\vct a_i^\perp\in\range(\mat A)$
and $\langle\vct a_i^\perp, \vct a_{i'}\rangle = 0$ for all $i'\neq i$.
Consequently,
$$
\sigma_m(\mat A) \leq \inf_{\vct u\in\range(\mat A)}\frac{\|\mat A\vct u\|_2}{\|\vct u\|_2}
\leq \frac{\|\mat A\vct a_i^\perp\|_2}{\|\vct a_i^\perp\|_2}
= \frac{\langle\vct a_i,\vct a_i^\perp\rangle}{\|\vct a_i^\perp\|_2}
= \frac{\|\vct a_i^\perp\|_2^2}{\|\vct a_i^\perp\|_2} = \|\vct a_i^\perp\|_2.
$$
\end{proof}

We next present two key lemmas.
The first lemma, Lemma \ref{lem_angdist_ub},
shows that the estimated subspace $\hat{\mathcal S}$ from noisy inputs is a good approximation
the underlying subspace $\mathcal S^{(\ell)}$ as long as the restricted eigenvalue assumption holds
and exactly $d$ points from the same subspace are used to construct $\hat{\mathcal S}$.

\begin{lem}
Fix $\ell\in\{1,\cdots, L\}$.
Suppose $\hat{\mathcal S}$ is the range of a subset of points $\mat Y_d\subseteq\mat Y^{(\ell)}$
containing exactly $d$ noisy data points belonging to the $\ell$th subspace.
Let $\mathcal S^{(\ell)}$ be the ground-truth subspace; i.e., $\vct x_1^{(\ell)},\cdots,\vct x_{N_\ell}^{(\ell)}\in\mathcal S^{(\ell)}$.
Under Assumption \ref{asmp_re} we have
\begin{equation}
d(\hat{\mathcal S}, \mathcal S^{(\ell)}) \leq \frac{2d\xi^2}{\sigma_\ell^2}.
\label{eq_angdist_ub}
\end{equation}
\label{lem_angdist_ub}
\end{lem}
\begin{proof}
Suppose $\mat Y_d=(\vct y_{i_1}^{(\ell)}, \cdots, \vct y_{i_d}^{(\ell)})$ and $\mat X_d=(\vct x_{i_1}^{(\ell)}, \cdots, \vct x_{i_d}^{(\ell)})$.
By the noise model $\|\mat Y_d-\mat X_d\|_F^2 = \sum_{j=1}^d{\|\vct\varepsilon_{i_j}\|_2^2} \leq d\xi^2$.
On the other hand, by Assumption \ref{asmp_re} we have $\sigma_d(\mat X_d) \geq \sigma_\ell$.
Wedin's theorem (Lemma \ref{lem_wedin} in Appendix \ref{appsec:matpert}) then yields the lemma.
\end{proof}

In Lemma \ref{lem_support_lb} we show that if the restricted eigenvalue assumption holds
and the regularization parameter $\lambda$ is in a certain range,
the optimal solution to the Lasso problem in Eq. (\ref{eq_noisy_ssc}) has at least $d$ nonzero coefficients,
which lead to $|V_r|\geq d+1$ for every connected component $V_r$ in the similarity graph constructed in Algorithm \ref{alg_noisy_ssc}.
Lemma \ref{lem_support_lb} is a natural extension to the fact that at least $d$ points should be used to reconstruct a certain data point for noiseless inputs,
if the data matrix $\mat X$ is in general position.

\begin{lem}
Assume Assumption \ref{asmp_re} and the self-expressiveness property hold.
For each $i\in\{1,\cdots,N\}$, $\|\vct c_i\|_0\geq d$ if the regularization parameter $\lambda$ satisfies
\begin{equation}
2\xi(1+\xi)^2(1+1/\rho_\ell)
< \lambda
< \frac{\rho_\ell\sigma_\ell}{2},\quad \ell=1,\cdots,L.
\label{eq_lambda_range}
\end{equation}
\label{lem_support_lb}
\end{lem}
\begin{proof}
Because the self-expressiveness property holds,
we assume without loss of generality that
the support set of $\vct c_i$ with $\|\vct c_i\|_0=t$ is $\{\vct y_1^{(\ell)}, \cdots, \vct y_{t}^{(\ell)}\}$.
Assume by way of contradiction that $\|\vct c_i\|_0 < d$ and define $\vct y^\perp = \vct y_i^{(\ell)} - \sum_{j=1}^{d-1}{c_{i,j}\vct y_j^{(\ell)}}$,
where $c_{i,1}, \cdots, c_{i,d-1}$ contain all nonzero coefficients
\footnote{Some coefficients in $c_{i_1},\cdots,c_{i,d-1}$ might be zero because $\|\vct c_i\|_1$ could be smaller than $d-1$.}
 in $\vct c_i$.
Since $\vct c_i$ is optimal, the following must hold for every $\vct y_{i'}^{(\ell)}$ with $i'\neq i$:
\begin{equation}
\argmin_{c\in\mathbb R}\left\{\|\vct y^\perp - c\vct y_{i'}^{(\ell)}\|_2^2 + 2\lambda |c|\right\} = 0.
\label{eq_lasso_extra}
\end{equation}
To see the necessity of Eq. (\ref{eq_lasso_extra}), 
note that the optimal solution to Eq. (\ref{eq_lasso_extra}) $c^*\neq 0$ implies
$$\|\vct y_i^{(\ell)} - \mat Y_{-i}^{(\ell)}\tilde{\vct c}_i\|_2^2 + 2\lambda\|\tilde{\vct c}_i\|_1
\leq \|\vct y^{\perp} - c^*\vct y_{i'}^{(\ell)}\|_2^2 + 2\lambda|c^*| + 2\lambda\|\vct c_i\|_1
< \|\vct y^\perp\|_2^2 + 2\lambda\|\vct c_i\|_1
= \|\vct y_i^{(\ell)}-\mat Y_{-i}^{(\ell)}\vct c_i\|_2^2 + 2\lambda\|\vct c_i\|_1,$$
where $\tilde{\vct c}_i = \vct c_i + c^*\cdot \vct e_{i'}$.
This contradicts the optimality of $\vct c_i$ with respect to Eq. (\ref{eq_noisy_ssc}).

By optimality conditions, Eq. (\ref{eq_lasso_extra}) implies $|\langle\vct y^\perp,\vct y_{i'}^{(\ell)}\rangle| \leq\lambda$.
In the remainder of the proof we will show that under the assumptions made in Lemma \ref{lem_support_lb},
$|\langle\vct y^\perp,\vct y_{i'}^{(\ell)}\rangle| > \lambda$, which results in a contradiction.

In order to lower bound $|\langle\vct y^\perp,\vct y_{i'}^{(\ell)}\rangle|$ we first bound the noiseless version of the inner product
$|\langle\vct x^\perp,\vct x_{i'}^{(\ell)}\rangle|$, where $\vct x^\perp = \vct x_i^{(\ell)} - \sum_{j=1}^{d-1}{c_{i,j}\vct x_j^{(\ell)}}$.
A key observation is that $\vct x^\perp\in\mathcal S^{(\ell)}$ and hence by Proposition \ref{prop_innerprod_lb} and \ref{prop_perp_norm} the following
chain of inequality holds for any $\vct x_{i'}^{(\ell)}$ with $i'\neq i$:
\begin{equation}
\big|\langle\vct x^\perp,\vct x_{i'}^{(\ell)}\rangle\big|
\geq \rho_\ell\|\vct x^\perp\|_2
\geq \rho_\ell\left\|\vct x_i^{(\ell)} - \mathcal P_{\span(\vct x_{1:d-1}^{(\ell)})}(\vct x_i^{(\ell)})\right\|_2
\geq \rho_\ell\sigma_\ell.
\label{eq_x_lb}
\end{equation}

Our next objective is to upper bound the inner product perturbation $|\langle\vct y^\perp,	\vct y_{i'}^{(\ell)}\rangle-\langle\vct x^\perp,\vct x_{i'}^{(\ell)}\rangle|$
and subsequently obtain a lower bound on $|\langle\vct y^\perp,\vct y_{i'}^{(\ell)}\rangle|$.
Note that
$$
\langle\vct y^\perp,\vct y_{i'}^{(\ell)}\rangle = \langle\vct x^\perp,\vct x_{i'}^{(\ell)}\rangle + \langle\vct y^\perp-\vct x^\perp,\vct x_{i'}^{(\ell)}\rangle
+ \langle\vct x^\perp, \vct y_{i'}^{(\ell)}-\vct x_{i'}^{(\ell)}\rangle + \langle\vct y^\perp-\vct x^\perp, \vct y_{i'}^{(\ell)}-\vct x_{i'}^{(\ell)}\rangle;
$$
therefore,
\begin{equation}
\big|\langle\vct y^\perp,\vct y_{i'}^{(\ell)}\rangle - \langle\vct x^\perp,\vct x_{i'}^{(\ell)}\rangle\big|
\leq \|\vct y^\perp-\vct x^\perp\|\|\vct x_{i'}^{(\ell)}\| + \|\vct y^\perp\|\|\vct y_{i'}^{(\ell)}-\vct x_{i'}^{(\ell)}\|
\leq \|\vct y^\perp-\vct x^\perp\|_2 + \xi\|\vct y^\perp\|_2.
\label{eq_pert_ub}
\end{equation}

In order to upper bound $\|\vct y^\perp\|_2$ and $\|\vct y^\perp-\vct x^\perp\|_2$,
note that by definition $\|\vct y^\perp\|_2 = \|\vct y_1^{(\ell)} - \sum_{j=2}^d{c_{ij}\vct y_j^{(\ell)}}\|_2 \leq (1+\|\vct c_i\|_1)(1+\xi)$
and $\|\vct y^\perp-\vct x^\perp\|_2 = \|\vct\varepsilon_1^{(\ell)} - \sum_{j=2}^d{c_{ij}\vct y_j^{(\ell)}}\|_2 \leq \xi(1+\|\vct c_i\|_1)$.
Hence we only need to upper bound $\|\vct c_i\|_1$, which can be done by the following argument due to the optimality of $\vct c_i$:
By arguments on page 21 in \cite{wang2015noisy}, the following upper bound on $\|\vct c_i\|_1$ is proven:
\begin{equation}
\|\vct c_i\|_1 \leq \frac{1}{\rho_\ell} + \frac{\xi^2}{\lambda}\left(1+\frac{1}{\rho_\ell}\right)^2.
\label{eq_ci_ub_prelim}
\end{equation}
The lower bound on $\lambda$ in Eq. (\ref{eq_lambda_range}) implies that $\xi < \lambda(1+1/\rho_\ell)$.
Plugging this upper bound into Eq. (\ref{eq_ci_ub_prelim}) we obtain
\begin{equation}
\|\vct c_i\|_1 \leq 1/\rho_\ell + \xi(1+1/\rho_\ell) \leq (1+\xi)(1+1/\rho_\ell),
\end{equation}
which eliminates the dependency on $\lambda$.
We now substitute the simplified upper bound on $\|\vct c_i\|_1$ into the upper bound for $\|\vct y^\perp\|_2$, $\|\vct y^\perp-\vct x^\perp\|_2$
and get
\begin{equation}
\|\vct y^\perp\|_2 \leq (1+\xi)^2(1+1/\rho_\ell);\quad
\|\vct y^\perp-\vct x^\perp\|_2 \leq \xi(1+\xi)(1+1/\rho_\ell).
\label{eq_perp_ub_real}
\end{equation}

Combining Eq. (\ref{eq_x_lb}), (\ref{eq_pert_ub}) and (\ref{eq_perp_ub_real}) we obtain the following lower bound on $|\langle\vct y^\perp,\vct y_{i'}^{(\ell)}\rangle|$:
\begin{equation}
\big|\langle\vct y^\perp,\vct y_{i'}^{(\ell)}\rangle\big|
\geq \rho_\ell\sigma_\ell - 2\xi(1+\xi)^2(1+1/\rho_\ell) \geq \frac{1}{2}\rho_\ell\sigma_\ell,
\end{equation}
where the last inequality is due to the assumption that $2\xi(1+\xi)^2(1+1/\rho_\ell) < \frac{1}{2}\rho_\ell\sigma_\ell$ implied by Eq. (\ref{eq_lambda_range}).
Finally, since $\frac{1}{2}\rho_\ell\sigma_\ell > \lambda$ as assumed in Eq. (\ref{eq_lambda_range}),
we have $|\langle\vct y^\perp,\vct y_{i'}^{(\ell)}\rangle| > \lambda$,
which results in the desired contradiction.
\end{proof}

Finally, Theorem \ref{thm_noisy_ssc} is a simple consequence of Lemma \ref{lem_angdist_ub} and \ref{lem_support_lb}
because under the conditions of Lemma \ref{lem_support_lb}, every component $V_r$ will have at least $d$ data points.
Define $\mu_\epsilon=\sqrt{2d\xi^2/\min_\ell\sigma_\ell^2}$.
Lemma \ref{lem_angdist_ub} implies that $d(\hat{\mathcal S}_{(r)}, \hat{\mathcal S}_{(r')})\leq\mu_\epsilon$
if $V_r$ and $V_{r'}$ belong to the same cluster.
On the other hand, by the separation condition in Eq. (\ref{eq_subspace_sep}) and Lemma \ref{lem_angdist_ub},
if $V_r$ and $V_{r'}$ belong to different clusters we would have $d(\hat{\mathcal S}_{(r)},\hat{\mathcal S}_{(r')}) > \mu_\epsilon$.
Therefore, the single-linkage clustering procedure in Algorithm \ref{alg_noisy_ssc} will eventually merge estimated subspaces correectly.

\section{MATRIX PERTURBATION THEOREMS}\label{appsec:matpert}

\begin{lem}[Wedin's theorem; Theorem 4.1, pp. 260 in \cite{mat-pert-theory}]
Let $\mat A,\mat E\in\mathbb R^{m\times n}$ be given matrices with $m\geq n$.
Let $\mat A$ have the following singular value decomposition
$$
\left[\begin{array}{c}\mat U_1^\top\\ \mat U_2^\top\\ \mat U_3^\top\end{array}\right]
\mat A
\left[\begin{array}{cc} \mat V_1& \mat V_2\end{array}\right]
= \left[\begin{array}{cc} \mat\Sigma_1& \mat 0\\ \mat 0& \mat\Sigma_2\\ \mat 0& \mat 0\end{array}\right],
$$
where $\mat U_1,\mat U_2,\mat U_3,\mat V_1,\mat V_2$ have orthonormal columns and $\mat\Sigma_1$ and $\mat\Sigma_2$ are diagonal matrices.
Let $\widetilde{\mat A} = \mat A+\mat E$ be a perturbed version of $\mat A$ and
$(\widetilde{\mat U}_1,\widetilde{\mat U}_2,\widetilde{\mat U}_3,\widetilde{\mat V}_1,\widetilde{\mat V}_2,\widetilde{\mat\Sigma}_1,\widetilde{\mat\Sigma}_2)$
be analogous singular value decomposition of $\widetilde{\mat A}$.
Let $\mat\Phi$ be the matrix of canonical angles between $\range(\mat U_1)$ and $\range(\widetilde{\mat U}_1)$
and $\mat\Theta$ be the matrix of canonical angles between $\range(\mat V_1)$ and $\range(\widetilde{\mat V}_1)$.
If there exists $\delta>0$ such that
$$
\min_{i,j}\big|[\mat\Sigma_1]_{i,i} - [\mat\Sigma_2]_{j,j}\big| > \delta\;\;\text{and}\;\;
\min_{i}\big|[\mat\Sigma_1]_{i,i}\big| > \delta,
$$
then
$$
\|\sin\mat\Phi\|_F^2 + \|\sin\mat\Theta\|_F^2 \leq \frac{2\|\mat E\|_F^2}{\delta^2}.
$$
\label{lem_wedin}
\end{lem}

\end{appendices}

\end{document}